\documentclass[10pt,journal,compsoc]{IEEEtran}

\ifCLASSOPTIONcompsoc
  \usepackage[nocompress]{cite}
\else
  \usepackage{cite}
\fi
\ifCLASSINFOpdf
  \usepackage[pdftex]{graphicx}
\else
  \usepackage[dvips]{graphicx}
\fi
\usepackage{amsmath}
\usepackage{algorithm}
\usepackage{algorithmic}
\usepackage{url}
\newcommand\MYhyperrefoptions{bookmarks=true,bookmarksnumbered=true,
pdfpagemode={UseOutlines},plainpages=false,pdfpagelabels=true,
colorlinks=true,linkcolor={black},citecolor={black},urlcolor={black},
pdftitle={Self-Adaptive Training: Bridging Supervised and Self-Supervised Learning},
pdfsubject={Machine Learning},
pdfauthor={Lang Huang},
pdfkeywords={Self-Supervised Learning, Generalization, Robust Learning under Noise}}
\ifCLASSINFOpdf
\usepackage[\MYhyperrefoptions,pdftex]{hyperref}
\else
\usepackage[\MYhyperrefoptions,breaklinks=true,dvips]{hyperref}
\usepackage{breakurl}
\fi
\usepackage[utf8]{inputenc} 
\usepackage[T1]{fontenc}    
\usepackage{booktabs}       
\usepackage{multirow}
\usepackage{graphicx}
\usepackage{subcaption}
\usepackage{bm}
\usepackage{bbm}
\usepackage{xcolor}
\usepackage{amsfonts}
\usepackage{amssymb}
\usepackage{enumitem}
\usepackage{amsthm}

\newtheorem{proposition}{Proposition}
\newtheorem{corollary}{Corollary}

\newcommand{\p}{\bm{p}}
\newcommand{\x}{\bm{x}}
\newcommand{\y}{\bm{y}}

\newcommand{\X}{\mathbf{X}}
\newcommand{\A}{\mathbf{A}}

\newcommand{\D}{\mathbf{D}}
\newcommand{\I}{\mathbf{I}}

\newcommand{\V}{\mathbf{V}}
\DeclareMathOperator{\argmax}{argmax}
\newcommand{\norm}[1]{\left\lVert#1\right\rVert}

\begin{document}
%
\title{Self-Adaptive Training: Bridging Supervised and Self-Supervised Learning}
%
%
%
%

\author{Lang~Huang,
        Chao~Zhang,
        and~Hongyang~Zhang
\IEEEcompsocitemizethanks{
\IEEEcompsocthanksitem The material in this paper was presented in part at Thirty-fourth Conference on Neural Information Processing Systems, December 2020~\cite{huang2020self}.
\IEEEcompsocthanksitem L.~Huang is with the Department of Information \& Communication Engineering, The University of Tokyo, Tokyo, Japan.
\IEEEcompsocthanksitem C.~Zhang is with the Key Laboratory of Machine Perception (MOE), School of Intelligence Science and Technology, Peking University, Beijing, China.
\IEEEcompsocthanksitem H.~Zhang is with the David R. Cheriton School of Computer Science, University of Waterloo, Waterloo, Canada.
\IEEEcompsocthanksitem E-mail address: laynehuang@pku.edu.cn, c.zhang@pku.edu.cn, and hongyang.zhang@uwaterloo.ca}
\thanks{Manuscript received Month Date, Year; revised Month Date, Year.}}

\markboth{Journal of \LaTeX\ Class Files,~Vol.~XX, No.~XX, September~2022}%
{Shell \MakeLowercase{\textit{et al.}}: Self-Adaptive Training: Bridging Supervised and Self-Supervised Learning}
%



\IEEEtitleabstractindextext{%
\begin{abstract}
We propose self-adaptive training---a unified training algorithm that dynamically calibrates and enhances training processes by model predictions without incurring an extra computational cost---to advance both supervised and self-supervised learning of deep neural networks. We analyze the training dynamics of deep networks on training data that are corrupted by, e.g., random noise and adversarial examples. Our analysis shows that model predictions are able to magnify useful underlying information in data and this phenomenon occurs broadly even in the absence of \emph{any} label information, highlighting that model predictions could substantially benefit the training processes: self-adaptive training improves the generalization of deep networks under noise and enhances the self-supervised representation learning. The analysis also sheds light on understanding deep learning, e.g., a potential explanation of the recently-discovered double-descent phenomenon in empirical risk minimization and the collapsing issue of the state-of-the-art self-supervised learning algorithms. Experiments on the CIFAR, STL, and ImageNet datasets verify the effectiveness of our approach in three applications: classification with label noise, selective classification, and linear evaluation. To facilitate future research, the code has been made publicly available at \url{https://github.com/LayneH/self-adaptive-training}.
\end{abstract}

\begin{IEEEkeywords}
Deep Learning, Supervised Learning, Self-Supervised Learning, Generalization, Robust Learning under Noise.
\end{IEEEkeywords}
}

\maketitle

\IEEEdisplaynontitleabstractindextext

%
\IEEEpeerreviewmaketitle

\ifCLASSOPTIONcompsoc
\IEEEraisesectionheading{\section{Introduction}\label{sec:introduction}}
\else
\section{Introduction}
\label{sec:introduction}
\fi

%
%
%
%

 
\IEEEPARstart{D}{eep} neural networks have received significant attention in machine learning and computer vision, in part due to their impressive performance achieved by supervised learning approaches in the ImageNet challenge. With the help of massive labeled data, deep neural networks advance the state-of-the-art to an unprecedented level on many fundamental tasks, such as image classification~\cite{simonyan2014very,he2016deep}, object detection~\cite{girshick2014rich}, and semantic segmentation~\cite{long2015fully}. However, data acquisition is notoriously costly, error-prone, and even infeasible in certain cases. Furthermore, deep neural networks suffer significantly from overfitting in these scenarios. On the other hand, the great success of self-supervised pre-training in natural language processing (e.g., GPT~\cite{radford2018improving,radford2019language,brown2020language} and BERT~\cite{devlin2019bert}) highlights that learning universal representations from unlabeled data can be even more beneficial for a broad range of downstream tasks.

Regarding this, much effort has been devoted to learning representations without human supervision for computer vision. Several recent studies show promising results and largely close the performance gap between supervised and self-supervised learning. To name a few, the contrastive learning approaches~\cite{he2020momentum,chen2020simple,caron2020unsupervised} solve the instance-wise discrimination task~\cite{wu2018unsupervised} as a proxy objective of representation learning. Extensive studies demonstrate that self-supervisedly learned representations are generic and even outperform the supervised pre-trained counterparts when they are fine-tuned on certain downstream tasks.

Our work advances both supervised learning and self-supervised learning settings. Instead of designing two distinct algorithms for each learning paradigm separately, in this paper, we explore the possibility of a unified algorithm that bridges supervised and self-supervised learning. Our exploration is based on two observations on the learning of deep neural networks.

\begin{figure*}[t]
    \centering
    \begin{subfigure}{\textwidth}
        \centering
        \includegraphics[width=\textwidth]{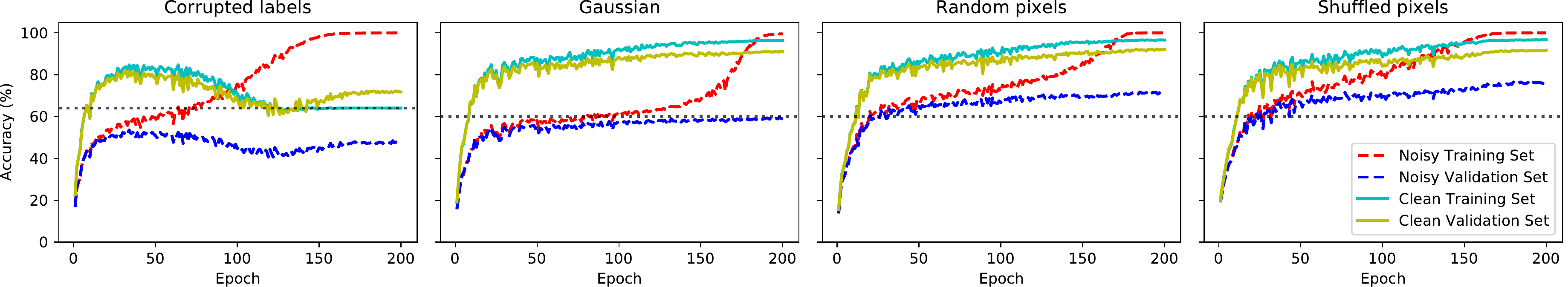}
        \caption{Accuracy curves of the model trained by ERM.}
        \label{fig:ce_acc_curve}
    \end{subfigure}
    \vskip 0.1in
    \begin{subfigure}{\textwidth}
        \centering
        \includegraphics[width=\textwidth]{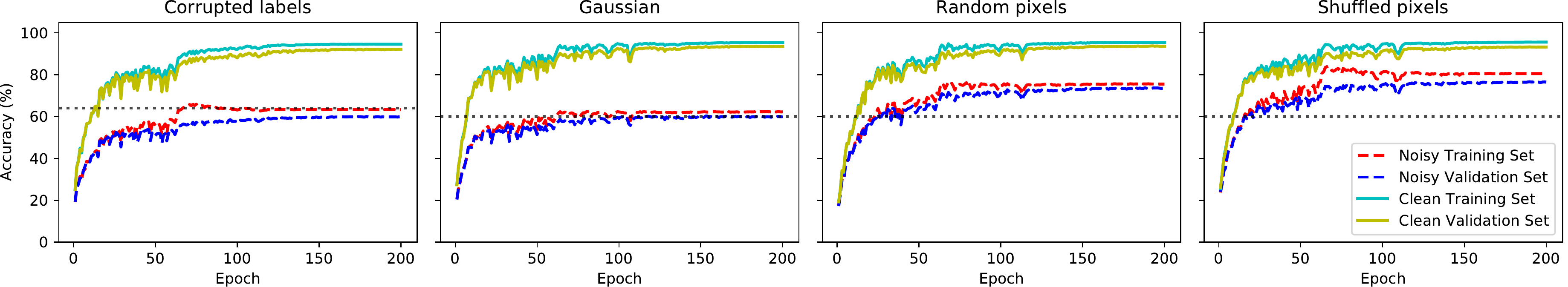}
        \caption{Accuracy curves of the model trained by our method.}
        \label{fig:wsc_acc_curve}
    \end{subfigure}
    \caption{
    Accuracy curves of models on the CIFAR10 dataset with 40\% of corrupted data.
    All models can only have access to the noisy training set (corresponding to the red dashed curve) during training and are tested on the other three sets for illustration.
    The horizontal dotted line displays the percentage of clean data in the training sets.
    It shows that while ERM suffers from overfitting to noise, our method enjoys improved generalization and higher validation accuracy
    (e.g., the improvement is larger than 10\% in the left-most column).
    }
    \label{fig:acc_curve}
\end{figure*}

\medskip
\noindent\textbf{Observation I:}\quad
Under various \emph{supervised learning} settings with noisy data, deep neural networks are able to magnify useful underlying information of data by their predictions.

\medskip
\noindent\textbf{Observation II:}\quad
Observation I extends to the extreme noise where the supervised signals are completely random, which is equivalent to \emph{self-supervised learning}.

These two observations indicate that model predictions can magnify useful information in data, which further indicates that incorporating predictions into the training processes could significantly benefit model learning. With this in mind, we propose \emph{self-adaptive training}---a carefully designed approach that dynamically uses model predictions as a guiding principle in the design of training algorithms---that bridges supervised learning and self-supervised learning in a unified framework.
Our approach is conceptually simple yet calibrates and significantly enhances the learning of deep models in multiple ways.

\subsection{Summary of our contributions}
\label{sec:contribution}

Self-adaptive training sheds light on understanding and improving the learning of deep neural networks.

\begin{itemize}
    \item We analyze the Empirical Risk Minimization (ERM) training processes of deep models on four kinds of corruptions (see Fig.~\ref{fig:ce_acc_curve}). We describe the failure scenarios of ERM and observe that useful information from data has been distilled to model predictions in the first few epochs. We show that this phenomenon occurs broadly even in the absence of any label information (see Fig.~\ref{fig:ssl_acc_at_training}). These insights motivate us to propose self-adaptive training---a unified training algorithm for both supervised and self-supervised learning---to improve the learning of deep neural networks by dynamically incorporating model predictions into training, requiring no modification to existing network architecture and incurring almost no extra computational cost.

    \item We show that self-adaptive training improves the generalization of deep networks under both label-wise and instance-wise random noise (see Fig.~\ref{fig:acc_curve}~and~\ref{fig:gen_clean_errs}). Besides, self-adaptive training exhibits a single-descent error-capacity curve (see Fig.~\ref{fig:double_descent}). This is in sharp contrast to the recently-discovered double-descent phenomenon in ERM, which might be a result of overfitting to noise. Moreover, while adversarial training may easily overfit adversarial noise, our approach mitigates the overfitting issue and improves the adversarial accuracy by $\sim$3\% over the state-of-the-art (see Fig.~\ref{fig:robust_acc}).

    \item Self-adaptive training questions and alleviates the dependency of recent self-supervised algorithms on the dominant training mechanism that typically involves multiple augmented views of the same images at each training step: self-adaptive training achieves remarkable performance despite requiring only a single view of each image for training, which significantly reduces the heavy cost of data pre-processing and model training on extra views.

\end{itemize}

Self-adaptive training has three applications and advances the state-of-the-art by a significant gap.
\begin{itemize}
    \item Learning with noisy labels, where the goal is to improve the performance of deep networks on clean test data in the presence of training label noise. On the CIFAR datasets, our approach obtains up to 9\% absolute classification accuracy improvement over the state-of-the-art. On the ImageNet dataset, our approach improves over ERM by 3\% under 40\% noise rate.
    
    \item Selective classification, which trades prediction coverage for classification accuracy. Our approach achieves up to 50\% relative improvement over the state-of-the-art on three datasets under various coverage rates.
    
    \item Linear evaluation, which evaluates the representations of a self-supervised pre-trained model using a linear classifier. Our approach performs on par with or even better than the state-of-the-art on various datasets.
    In particular, on the ImageNet dataset, self-adaptive training achieves 72.8\% top1 linear accuracy using only 200 pre-training epochs, surpassing the other methods by 2.2\% in absolute under the same setting.
\end{itemize}
\section{Self-Adaptive Training}
\label{sec:sat}

\subsection{Blessing of model predictions}
\label{sec:Corrupted data}

\noindent\textbf{On corrupted data}\quad 
Recent works~\cite{zhang2016understanding,NIPS2019_9336} cast doubt on the ERM training: techniques such as uniform convergence might be unable to explain the generalization of deep neural networks because ERM easily overfits the training data even though the training data are partially or completely corrupted by random noise.
To take a closer look at this phenomenon, we conduct the experiments on the CIFAR10 dataset~\cite{krizhevsky2009cifar}, splitting the original training data into a training set (consists of the first 45,000 data pairs) and a validation set (consists of the last 5,000 data pairs). We measure four random noise schemes according to prior work~\cite{zhang2016understanding}, where the data are \emph{partially} corrupted with probability $p$:
1)~\emph{Corrupted labels}. Labels are assigned uniformly at random;
2)~\emph{Gaussian}. Images are replaced by random Gaussian samples with the same mean and standard deviation as the original image distribution;
3)~\emph{Random pixels}. Pixels of each image are shuffled using independent random permutations;
4)~\emph{Shuffled pixels}. Pixels of each image are shuffled using a fixed permutation pattern. We consider the performance on both the noisy and the clean sets (i.e., the original uncorrupted data), while the models can only have access to the noisy training sets.

Fig.~\ref{fig:ce_acc_curve} displays the accuracy curves of ERMs that are trained on the noisy training sets under four kinds of random corruptions: ERM easily overfits noisy training data and achieves nearly perfect training accuracy. However, the four subfigures exhibit very different generalization behaviors which are indistinguishable if we only look at the accuracy curve on the noisy training set (the red curve).
In Fig.~\ref{fig:ce_acc_curve}, the accuracy increases in the early stage, and the generalization errors grow quickly only after a certain number of epochs. Intuitively, early-stopping improves the generalization in the presence of label noise (see the first column in Fig. \ref{fig:ce_acc_curve}); however, it remains unclear how to properly identify such an epoch without using validation data. Moreover, the early-stop mechanism may significantly hurt the performance on the clean validation sets, as we can see in the last three columns of Fig.~\ref{fig:ce_acc_curve}.
Our approach is motivated by the failure scenarios of ERM and goes beyond ERM. We begin by making the following observations in the leftmost subfigure of
Fig.~\ref{fig:ce_acc_curve}: the peak of the accuracy curve on the clean training set (>80\%) is much higher than the percentage of clean data in the noisy training set ($\sim$60\%). This finding was also previously reported by~\cite{rolnick2017deep,guan2018said,li2019gradient} under label corruption and suggested that model predictions might be able to magnify useful underlying information in data. We confirm this finding and show that the pattern occurs under various kinds of corruption more broadly (see the last three subfigures of Fig.~\ref{fig:ce_acc_curve}).

\medskip
\noindent\textbf{On unlabelled data}\quad
\label{sec:sat_ssl_prelim}
We notice that supervised learning with 100\% noisy labels is equivalent to unsupervised learning if we simply discard the meaningless labels. Therefore, it would be interesting to analyze how deep models behave in such an extreme case. Here, we conduct experiments on the CIFAR10~\cite{krizhevsky2009cifar} dataset and consider two kinds of random noise as the training (real-valued) targets for deep learning models: 1) the output features of another model on the same training images, where the model is randomly initialized and then frozen; 2) random noise that is drawn i.i.d. from standard Gaussian distribution and then fixed.
The training of the deep model is then formulated as minimizing the mean square error between $\ell_2$-normalized model predictions and these two kinds of random noise. To monitor the training, we learn a linear classifier on the top of each model to evaluate its representation at each training epoch (see Appendix~\ref{sec:setup_online_lin_cls} for the detailed setup of this classifier).

\begin{figure}[t]
    \centering
    \includegraphics[width=.8\linewidth]{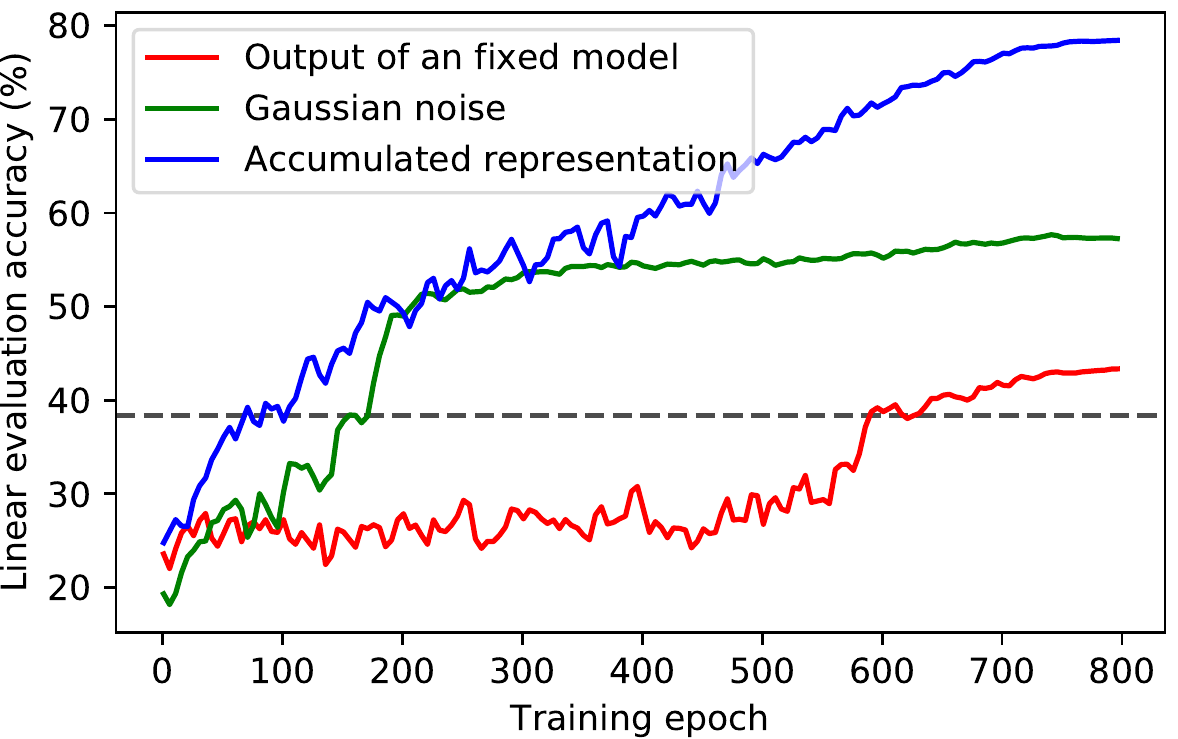}
    \caption{Linear evaluation accuracy on CIFAR10 of checkpoints of the model trained by fitting the noise as described in Sec.~\ref{sec:sat_ssl_prelim} (the red and green curves) and fitting the accumulated representation of the model (the blue curve). The horizontal dashed line indicates the accuracy of linear evaluation directly using a randomly initialized network. All the curves are smoothed for better demonstration.}
    \label{fig:ssl_acc_at_training}
\end{figure}

Fig.~\ref{fig:ssl_acc_at_training} shows the linear evaluation accuracy curves of models trained on these two kinds of noise, i.e., the green and red curves (the explanation of the blue curve is deferred to Sec.~\ref{sec:sat_ssl_instantiation}). We observe that, perhaps surprisingly, the model trained by predicting fixed random Gaussian noise (the green curve) achieves 57\% linear evaluation accuracy, which is substantially higher than the 38\% accuracy of a linear classifier trained on top of a randomly initialized network (the dashed horizontal line). This intriguing observation shows that deep neural networks are able to distill underlying information from data to its predictions, even when supervision signals are completely replaced by random noise. Furthermore, although predicting the output of another network can also improve the representation learning (the red curve), its performance is worse than that of the second scheme. We hypothesize that this might be a result of inconsistent training targets: the output of a network depends on the input training image, which is randomly transformed by the data augmentation at each epoch. This suggests that the consistency of training targets should also be taken into account in the design of the training algorithm.

\medskip
\noindent\textbf{Inspiration for our methodology}\quad
Based on our analysis, we propose a unified algorithm, \emph{Self-Adaptive Training}, for both supervised and self-supervised learning. Self-adaptive training incorporates model predictions to augment the training processes of deep models in a \emph{dynamic} and \emph{consistent} manner. Our methodology is general and very effective: self-adaptive training significantly improves (a) the generalization of deep models on the corrupted data, and (b) the representation learning of deep models without human supervision.

\subsection{Meta algorithm: Self-Adaptive Training}

Given a set of training images $\{\x_i\}_n$ and a deep network $f_{\theta}(\cdot)$ parametrized by $\theta$, our approach records training targets $\{\bm{t}_i\}_n$ for all data points accordingly. We first obtain the predictions of the deep network as
\begin{align}
    \p_i & = \rho(f_{\theta}(\x_i)),
\end{align}
where $\rho(\cdot)$ is a normalization function. Then, the training targets track all historical model predictions during training and are updated by Exponential-Moving-Average (EMA) scheme as
\begin{align}
\label{eq:meta_ema}
    \bm{t}_{i} & \leftarrow \alpha\times \bm{t}_{i} + (1 - \alpha)\times \p_i.
\end{align}
The EMA scheme in Equation~\eqref{eq:meta_ema} alleviates the instability issue of model predictions, smooths out $\bm{t}_i$ during the training process, and enables our algorithm to completely change the training labels if necessary. The momentum term $\alpha$ controls the weight of the model predictions.
Finally, we can update the weights $\theta$ of the deep network $f_{\theta}$ by Stochastic Gradient Descent (SGD) on the loss function $\mathcal{L}(\p_i, \bm{t}_i; f_{\theta})$ at each training iteration.

We summarize the meta algorithm of self-adaptive training in Algorithm~\ref{alg:sat_meta_alg}. The algorithm is conceptually simple, flexible, and has three components adapting to different learning settings: 1) the training \emph{targets initialization}; 2) \emph{normalization function} $\rho(\cdot)$; 3) \emph{loss function} $\mathcal{L}(\p_i, \bm{t}_i; f_{\theta})$. In the following sections, we will elaborate on the instantiations of these components for specific learning settings.

\begin{algorithm}
\small
\caption{Self-Adaptive Training}
\label{alg:sat_meta_alg}
\begin{algorithmic}[1]
    \REQUIRE Data $\{\x_i\}_n$, deep network $f_{\theta}(\cdot)$ parametrized by $\theta$, momentum term $\alpha$, normalization function $\rho(\cdot)$
    \STATE Initialize targets $\{\bm{t}_{i}\}_n$
    \REPEAT
      \STATE Fetch mini-batch data $\{(\x_i, \bm{t}_{i})\}_m$
      \FOR{$i=1$ {\bfseries to} $m$  (in parallel)}
        \STATE $\p_i = \rho(f_{\theta}(\x_i))$
        \STATE $\bm{t}_{i} \leftarrow \alpha\times \bm{t}_{i} + (1 - \alpha)\times \p_i$
      \ENDFOR
      \STATE Update $f_{\theta}$ by SGD on $\mathcal{L}(\p_i, \bm{t}_i; f_{\theta})$
    \UNTIL{end of training}
\end{algorithmic}
\end{algorithm}

\medskip
\noindent\textbf{Convergence analysis}\quad
\label{sec:convergence}
To simplify the analysis, we consider a linear regression problem with data $\{\x_i\}_n$, training targets $\{t_i\}_n$ and a linear model $f_{\bm{\theta}}(\x)=\x\bm{\theta}$, where $\x_i\in \mathbb{R}^{d}$, $t_i\in \mathbb{R}$ and $\theta\in\mathbb{R}^{d}$. Let $\X=[\x_1|\x_2|\cdots|\x_n]\in\mathbb{R}^{n\times d}$, $\bm{t}=[t_1|t_2|\cdots|t_n]\in\mathbb{R}^{n}$. Then the optimization for this regression problem (corresponding to $\min_{\bm{\theta}} \mathcal{L}(\p_i, \bm{t}_i; f_{\theta})$ in Algorithm~\ref{alg:sat_meta_alg}) can be written as 
\begin{align}
\label{eq:lin_reg}
    \arg\min_{\bm{\theta}} \norm{\X\bm{\theta} - \bm{t}}_2^2.
\end{align}

Let $\bm{\theta}^{(k)}$ and $\bm{t}^{(k)}$ be the model parameters and training targets at the $k$-th training step, respectively. Let $\eta$ denote the learning rate for gradient descent update. The Algorithm~\ref{alg:sat_meta_alg} alternatively minimizes the problem~\eqref{eq:lin_reg} over $\bm{\theta}^{(k)}$ and $\bm{t}^{(k)}$ as
\begin{align}
\begin{split}
\label{eq:theta_k}
    \bm{\theta}^{(k)} &= \bm{\theta}^{(k-1)} - \eta\nabla f_{\bm{\theta}^{(k-1)}}(\X) \\
     &= \bm{\theta}^{(k-1)} - \eta\X^{\intercal}(\X\bm{\theta}^{(k-1)} - \bm{t}^{(k-1)}),
\end{split}
\end{align}
\begin{align}
\begin{split}
\label{eq:t_k}
    \bm{t}^{(k)} &= \alpha\bm{t}^{(k-1)} + (1-\alpha)f_{\bm{\theta}^{(k)}}(\X) \\
     & = \alpha\bm{t}^{(k-1)} + (1-\alpha)\X\bm{\theta}^{(k)}.
\end{split}
\end{align}

\begin{proposition}
\label{prop:conver}
Let $d_{\mathrm{max}}$ be the maximal eigenvalue of the matrix $\X\X^{\intercal}$, if the learning rate $\eta < \frac{\alpha+1}{\alpha d_{\mathrm{max}}}$, then
\begin{align}
    \lim_{k\rightarrow\infty} \norm{\X\bm{\theta}^{(k)} - \bm{t}^{(k)}}_2^2 = 0, 
\end{align}
\end{proposition}
\begin{corollary}
\label{corol:conver_rate}
Under the same condition as in Proposition~\ref{prop:conver}, we have
\begin{align}
    \lim_{k\rightarrow\infty} \frac{\norm{\X\bm{\theta}^{(k+1)} - \bm{t}^{(k+1)}}_2^2}{\norm{\X\bm{\theta}^{(k)} - \bm{t}^{(k)}}_2^2} < 1,
\end{align}
\end{corollary}
Combining Proposition~\ref{prop:conver} and Corollary~\ref{corol:conver_rate}, we see that, with a proper learning rate for gradient descent, the optimization of problem~\eqref{eq:lin_reg} converges at least $Q$-linearly to 0.
The proofs of Proposition~\ref{prop:conver} and Corollary~\ref{corol:conver_rate} are presented in Appendix~\ref{sec:proof}.

\medskip
\noindent\textbf{Limitations}\quad
We note that one disadvantage of self-adaptive training is the extra storage cost of the training targets, which might be a problem when training on an extremely large dataset. However, we find this cost is actually moderate in most realistic cases. Take the large-scale ImageNet dataset~\cite{deng2009imagenet} as an example. The ImageNet consists of about 1.2 million images categorized into 1,000 classes. The storage of such vectors in the single precision format for the entire dataset requires $1.2 \times 10^6 \times 1000 \times 32$ bit $\approx 4.47$GB, which is reduced to $\sim$$1.12$GB under the self-supervised learning setting that records a $256$-d feature for each image. The cost is acceptable since modern GPUs usually have 20GB or more dedicated memory, e.g., NVIDIA Tesla A100 has 40GB of memory. Moreover, the vectors can be (a) sharded to all GPU devices~\cite{FairScale2021}, (b) stored on the CPU memory, or even (c) stored on the disk and loaded along with the images, to further reduce the storage cost.

\begin{figure*}[t]
\centering
\includegraphics[width=\textwidth]{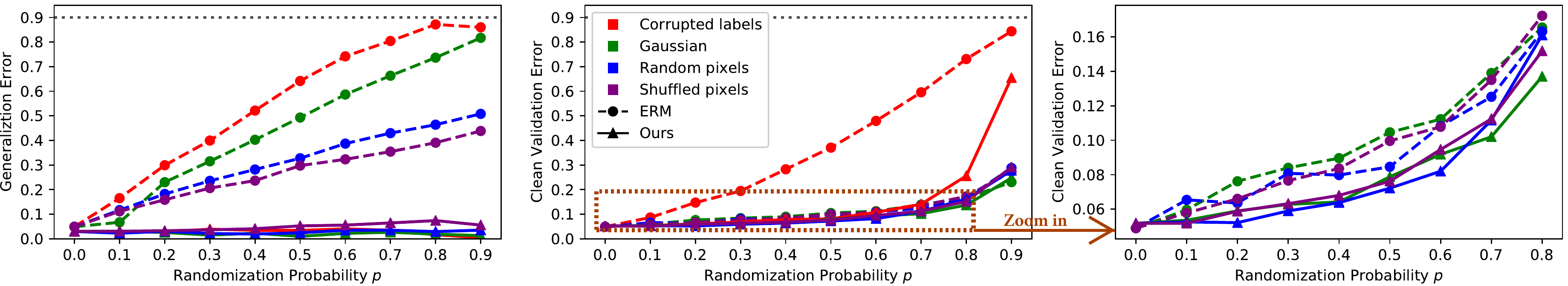}
\caption{
Generalization error and clean validation error under four kinds of random noise (represented by different colors) for ERM (the dashed curves) and our approach (the solid curves) on the CIFAR10 dataset. We zoom in on the dashed rectangle region and display it in the third column for a clear demonstration.
}
\label{fig:gen_clean_errs}
\end{figure*}

\section{Improved Generalization of Deep Models}
\label{sec:approach}

\subsection{Supervised Self-Adaptive Training}
\label{sec:sat_supervised_inst}

\noindent\textbf{Instantiation}\quad
We consider a $c$-class classification problem and denote the images by $\x_i$, labels by $\y_i\in \{0,1\}^c, \y_i^\intercal \mathbf{1}=1$. Given a data pair $(\x_i, \y_i)$, our approach instantiates the three components of meta Algorithm~\ref{alg:sat_meta_alg} for supervised learning as follows:
\begin{enumerate}
    \item \emph{Targets initialization}. Since the labels are provided, the training target $\bm{t}_i$ is directly initialized as $\bm{t}_i \leftarrow \y_i$.
    \item \emph{Normalization function}. We use the softmax function $\mathrm{softmax}(\bm{u})_j = \exp^{\bm{u}_j} / \sum_k \exp^{\bm{u}_k}$
    to normalize the model predictions into probability vectors $\bm{p}_i = \mathrm{softmax}(f_{\theta}(\x_i))$, such that $\bm{p}_i\in [0,1]^c, \bm{p}_i^\intercal \mathbf{1}=1$.
    \item \emph{Loss function}. Following the common practice in supervised learning, the loss function is implemented as the cross entropy loss between model predictions $\p_i$ and training targets $\bm{t}_i$, i.e., $\sum_j \bm{t}_{i,j}~ \mathrm{log}~\p_{i,j}$, where $\bm{t}_{i,j}$ and $\p_{i,j}$ represent the $j$th entry of $\bm{t}_i$ and $\p_{i}$, respectively.
\end{enumerate}

During the training process, we fix $\bm{t}_i$ in the first $\mathrm{E}_s$ training epochs and update the training targets $\bm{t}_i$ according to Equation~\eqref{eq:meta_ema} in each following training epoch.
The number of initial epochs $\mathrm{E}_s$ allows the model to capture informative signals in the data set and excludes ambiguous information that is provided by model predictions in the early stage of training.

\medskip
\noindent\textbf{Sample re-weighting}\quad
Based on the scheme presented above,
we introduce a simple yet effective sample re-weighting scheme on each sample.
Concretely, given training target $\bm{t}_i$, we set
\begin{equation}
\label{eq:sample_weight}
    w_i = \max_{j} ~\bm{t}_{i,j}.
\end{equation}
The sample weight $w_i \in [\frac{1}{c}, 1]$
reveals the labeling confidence of this sample.
Intuitively, all samples are treated equally in the first $\mathrm{E}_s$ epochs.
As the target $\bm{t}_i$ being updated, 
our algorithm pays less attention to potentially erroneous data
and learns more from potentially clean data.
This scheme also allows the corrupted samples to re-attain attention if they are confidently corrected.

\medskip
\noindent\textbf{Putting everything together}\quad
We use stochastic gradient descent to minimize:
\begin{equation}
\label{eq:overall_loss}
    \mathcal{L}(f) = -\frac{1}{\sum_i w_i}\sum_i w_i \sum_j \bm{t}_{i,j}~\mathrm{log}~\p_{i,j}
\end{equation}
during the training process.
Here, the denominator normalizes per sample weight and stabilizes the loss scale. We summarize the \emph{Supervised Self-Adaptive Training} and display the pseudocode in Algorithm~\ref{alg:approach}.
Intuitively, the optimal choice of hyper-parameter $\mathrm{E}_s$ should be related to the epoch where overfitting occurs, which is around the $60$th epoch according to Fig.~\ref{fig:ce_acc_curve}. 
For convenience, we directly fix the hyper-parameters $\mathrm{E}_s=60$, $\alpha=0.9$ by default if not specified. Experiments on the sensitivity of our algorithm to hyper-parameters are deferred to Sec. \ref{sec:exp_label_noise}. Our approach requires no modification to existing network architecture and incurs almost no extra computational cost.

\begin{algorithm}
\small
\caption{Supervised Self-Adaptive Training}
\label{alg:approach}
\begin{algorithmic}[1]
  \REQUIRE Data $\{(\x_i, \y_i)\}_n$; classifier $f_{\theta}$; initial epoch $\mathrm{E}_s = 60$; momentum term $\alpha = 0.9$
  \STATE Initialize training targets by $\{\bm{t}_{i}\}_n \leftarrow \{\y_{i}\}_n$
    \REPEAT
      \STATE Fetch mini-batch data $\{(\x_i, \bm{t}_{i})\}_m$ at current epoch $e$
      \FOR{$i=1$ {\bfseries to} $m$  (in parallel)}
        \STATE $\p_i = \mathrm{softmax}(f_{\theta}(\x_i))$
        \IF{$e > \mathrm{E}_s$}
            \STATE $\bm{t}_{i} \leftarrow \alpha\times \bm{t}_{i} + (1 - \alpha)\times \p_i$
        \ENDIF
        \STATE $w_{i} = \max_{j} ~\bm{t}_{i,j}$
      \ENDFOR
      \STATE $\mathcal{L}(\p_i, \bm{t}_i; f_{\theta}) = -\frac{1}{\sum_i w_i}\sum_i w_i \sum_j \bm{t}_{i,j}~ \mathrm{log}~\p_{i,j}$
      \STATE Update $f_{\theta}$ by SGD on $\mathcal{L}(\p_i, \bm{t}_i; f_{\theta})$
    \UNTIL{end of training}
\end{algorithmic}
\end{algorithm}

\medskip
\noindent\textbf{Methodology differences with prior work}\quad
Supervised self-adaptive training consists of two components: a) label correction; b) sample re-weighting. With the two components, our algorithm is robust to both instance-wise and label-wise noise and is ready to combine with various training schemes such as natural and adversarial training, without incurring multiple rounds of training. In contrast, the vast majority of works on learning from corrupted data follow a preprocessing-training scheme with an emphasis on the label-wise noise only: this line of research either discards samples based on the disagreement between noisy labels and model predictions~\cite{brodley1996identifying,brodley1999identifying,zhu2003eliminating,nguyen2019self}, or corrects noisy labels~\cite{bagherinezhad2018label,tanaka2018joint}; \cite{teng1999correcting} investigated a more generic approach that corrects both label-wise and instance-wise noise. However, their approach inherently suffers from extra computational overheads.
Besides, unlike the general scheme in robust statistics~\cite{rousseeuw2005robust} and other re-weighting methods~\cite{jiang2018mentornet,ren2018learning} that used an additional optimization step to update the sample weights, our approach directly obtains the weights based on accumulated model predictions and thus is much more efficient.

\subsection{Improved generalization under random noise}
We consider the noise scheme (including noise type and noise level) and model capacity as two factors that affect the generalization of deep networks under random noise. We analyze self-adaptive training by varying one of the two factors while fixing the other.

\medskip
\noindent\textbf{Varying noise schemes}\quad
We use ResNet-34~\cite{he2016deep} and rerun the same experiments as in Fig.~\ref{fig:ce_acc_curve} by replacing ERM with our approach. In Fig.~\ref{fig:wsc_acc_curve}, we plot the accuracy curves of models trained with our approach on four corrupted training sets and compare them with those in Fig.~\ref{fig:ce_acc_curve}.
We highlight the following observations.

\begin{itemize}
\item Our approach mitigates the overfitting issue in deep networks. The accuracy curves on noisy training sets (i.e., the red dashed curves in Fig.~\ref{fig:wsc_acc_curve}) nearly converge to the percentage of clean data in the training sets and do not reach perfect accuracy.

\item The generalization errors of self-adaptive training (the gap between the red and blue dashed curves in Fig.~\ref{fig:wsc_acc_curve}) are much smaller than Fig.~\ref{fig:ce_acc_curve}. We further confirm this observation by displaying the generalization errors of the models trained on the four noisy training sets under various noise rates in the leftmost subfigure of Fig.~\ref{fig:gen_clean_errs}. The generalization errors of ERM consistently grow as we increase the injected noise level. In contrast, our approach significantly reduces the generalization errors across all noise levels from 0\% (no noise) to 90\% (overwhelming noise).

\item The accuracy on the clean sets (cyan and yellow solid curves in Fig.~\ref{fig:wsc_acc_curve}) is monotonously increasing and converges to higher values than that of the counterparts in Fig.~\ref{fig:ce_acc_curve}. We also show the clean validation errors in the right two subfigures in Fig.~\ref{fig:gen_clean_errs}. The figures show that the error of self-adaptive training is consistently much smaller than that of ERM.
\end{itemize}

\begin{figure}[t]
\centering
\includegraphics[width=.8\linewidth]{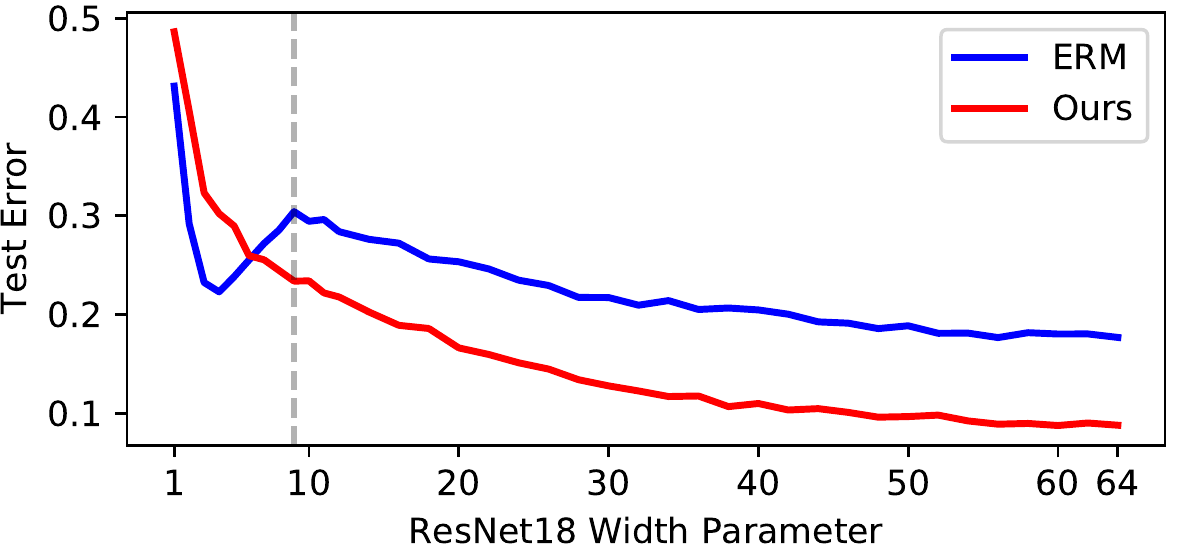}
\caption{
Double-descent ERM \emph{vs.} single-descent self-adaptive training on the error-capacity curve.
The model of width 64 corresponds to the standard ResNet-18. The vertical dashed line represents the interpolation threshold~\cite{belkin2018reconciling,nakkiran2019deep}.
}
\label{fig:double_descent}
\end{figure}

\medskip
\noindent\textbf{Varying model capacity}\quad
We notice that such analysis is related to a recently discovered intriguing phenomenon~\cite{opper1995statistical,opper2001learning,advani2017high,spigler2018jamming,belkin2018reconciling,geiger2019jamming,nakkiran2019deep} in modern machine learning models: as the capacity of the model increases, the test error initially decreases, then increases, and finally shows a second descent. This phenomenon is termed \emph{double descent} \cite{belkin2018reconciling} and has been widely observed in deep networks~\cite{nakkiran2019deep}. To evaluate the double-descent phenomenon on self-adaptive training, we follow exactly the same experimental settings as \cite{nakkiran2019deep}: we vary the width parameter of ResNet-18~\cite{he2016deep} and train the networks on the CIFAR10 dataset with 15\% of training labels being corrupted at random (details are given in Appendix~\ref{sec:setup_dd}).

Fig.~\ref{fig:double_descent} shows the test error curves. It shows that self-adaptive training overall achieves a much lower test error than that of ERM except when using extremely small models that underfit the training data. This suggests that our approach can improve the generalization of deep networks especially when the model capacity is reasonably large. Besides, we observe that the curve of ERM clearly exhibits the double-descent phenomenon, while the curve of our approach is monotonously decreasing as the model capacity increases. Since the double-descent phenomenon may vanish when label noise is absent~\cite{nakkiran2019deep}, our experiment indicates that this phenomenon may be a result of overfitting to noise and we can bypass it by a properly designed training process such as self-adaptive training.

\subsection{Improved generalization under adversarial noise}
Adversarial noise~\cite{szegedy2013intriguing} is different from random noise in that the noise is model-dependent and imperceptible to humans. We use the state-of-the-art adversarial training algorithm TRADES \cite{zhang2019theoretically} as our baseline to evaluate the performance of self-adaptive training under adversarial noise.
Algorithmically, TRADES minimizes
\begin{equation}
\label{eq:trades}
\mathbb{E}_{\x,\y}\Bigg\{ \mathrm{CE}(\p(\x), \y) + \max_{\|\widetilde{\x}-\x\|_\infty\le\epsilon}\mathrm{KL}(\p(\x), \p(\widetilde{\x}))/\lambda\Bigg\},
\end{equation}
where $\p(\cdot)$ is the model prediction, $\epsilon$ is the maximally allowed perturbation, CE stands for the cross entropy, KL stands for the Kullback–Leibler divergence, and the hyper-parameter $\lambda$ controls the trade-off between robustness and accuracy.
We replace the CE term in TRADES loss with our method.
The models are evaluated using robust accuracy $\frac{1}{n}\sum_i \mathbbm{1}\{ \argmax~\p(\widetilde{\x}_i)=\argmax~\y_i \}$, where adversarial examples $\widetilde{\x}$ are generated by white box $\ell_{\infty}$ AutoAttack~\cite{croce2020reliable} with $\epsilon$ = 0.031 (the evaluation of projected gradient descent attack \cite{madry2017towards} is given in Fig.~\ref{fig:robust_acc_pgd} of Appendix~\ref{sec:extra_exp}). We set the initial learning rate as 0.1 and decay it by a factor of 0.1 in epochs 75 and 90, respectively. 
We choose $1/\lambda=6.0$ as suggested by~\cite{zhang2019theoretically} and use $\mathrm{E}_s$ = 70, $\alpha$ = 0.9 for our approach.
Experimental details are given in Appendix~\ref{sec:setup_adv}.

We display the robust accuracy on the CIFAR10 test set after $\mathrm{E}_s$ = 70 epochs in Fig.~\ref{fig:robust_acc}. It shows that the robust accuracy of TRADES reaches its highest value around the epoch of the first learning rate decay (epoch 75) and decreases later, which suggests that overfitting might happen if we train the model without early stopping.
On the other hand, our method considerably mitigates the overfitting issue in the adversarial training and consistently improves the robust accuracy of TRADES by 1\%$\sim$3\%, which indicates that self-adaptive training can improve the generalization in the presence of adversarial noise.

\begin{figure}[t]
\centering
\includegraphics[width=.8\linewidth]{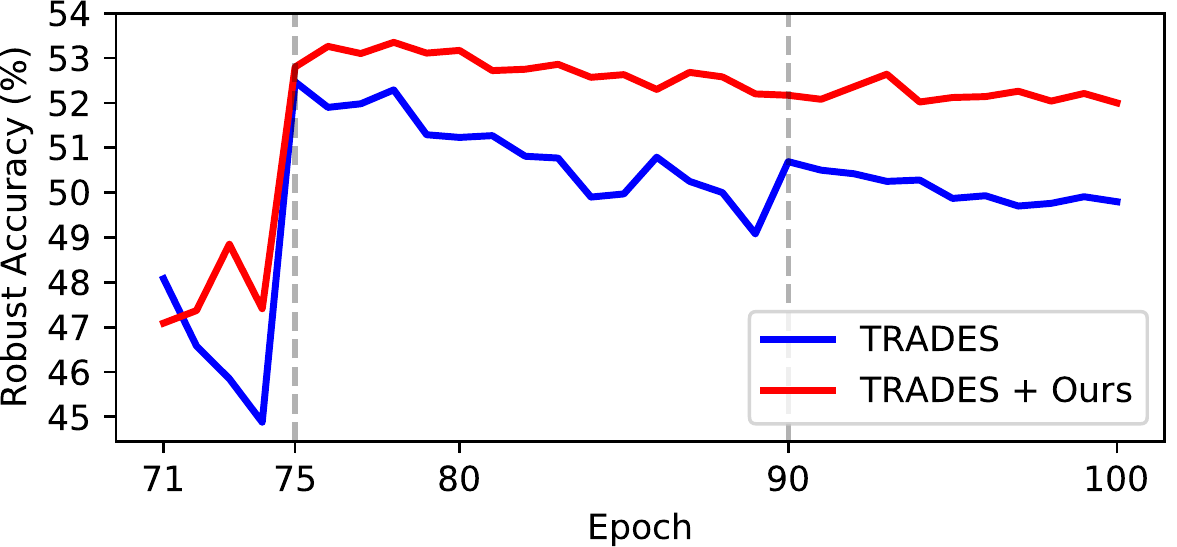}
\caption{
Robust Accuracy (\%) on the CIFAR10 test set under white box $\ell_\infty$ AutoAttack~\cite{croce2020reliable} ($\epsilon$=0.031). The plot of the first 70 epochs is omitted because the number of initial epochs $\mathrm{E}_s$ is set as 70. The vertical lines indicate learning rate decay. It shows that our method consistently improves TRADES.
}
\label{fig:robust_acc}
\end{figure}
\begin{table*}[t!]
\caption{Test Accuracy (\%) on the CIFAR datasets
with various levels of symmetric noise (SN) or asymmetric noise (AN) injected into the training labels.
We compare with previous works under exactly the same
experiment settings. It shows that, in all settings, self-adaptive training improves over the state-of-the-art by at least $1$\%, and sometimes the improvement is as significant as $9$\%. The best entries are bold-faced.
}
\label{tab:noisy_cls_all}
\begin{center}
\begin{small}
\setlength{\tabcolsep}{1.75mm}{
\begin{tabular}{llcccccccccccc}
\toprule
\multirow{2}{*}{Backbone} &  & \multicolumn{4}{c}{SN@CIFAR10} & \multicolumn{4}{c}{SN@CIFAR100}  & \multicolumn{2}{c}{AN@CIFAR10}  & \multicolumn{2}{c}{AN@CIFAR100} \\
 \cmidrule(l{3pt}r{3pt}){3-6} \cmidrule(l{3pt}r{3pt}){7-10} \cmidrule(l{3pt}r{3pt}){11-12} \cmidrule{13-14}
& Label Noise Rate  & 0.2 & 0.4 & 0.6 & 0.8 & 0.2 & 0.4 & 0.6 & 0.8 & 0.2 & 0.4 & 0.2 & 0.4  \\
\midrule
\multirow{10}{*}{ResNet-34} & ERM + Early Stop & 85.57 & 81.82 & 76.43 & 60.99 & 63.70 & 48.60 & 37.86 & 17.28 & 89.06 & 79.33 & 65.37 & 46.22 \\
& Label Smooth~\cite{szegedy2016rethinking}& 85.64 & 71.59 & 50.51 & 28.19 & 67.44 & 53.84 & 33.01 & 9.74 & - & - & - & - \\
& Forward $\hat{T}$~\cite{patrini2017making} & 87.99 & 83.25 & 74.96 & 54.64 & 39.19 & 31.05 & 19.12 & 8.99 & 89.09 & 83.55 & 42.46 & 34.44 \\
& Mixup~\cite{zhang2017mixup}                 & 93.58 & 89.46 & 78.32 & 66.32 & 69.31 & 58.12 & 41.10 & 18.77 & - & - & - & - \\
& Trunc $\mathcal{L}_q$~\cite{zhang2018gce}   & 89.70 & 87.62 & 82.70 & 67.92 & 67.61 & 62.64 & 54.04 & 29.60 & 89.33 & 76.74 & 66.59 & 47.22 \\
& Joint Opt~\cite{tanaka2018joint}            & 92.25 & 90.79 & 86.87 & 69.16 & 58.15 & 54.81 & 47.94 & 17.18 & - & - & - & - \\
& SCE~\cite{wang2019symmetric}                & 90.15 & 86.74 & 80.80 & 46.28 & 71.26 & 66.41 & 57.43 & 26.41 & 90.44 & 82.51 & 72.56 & 69.32 \\
& DAC~\cite{thulasidasan2019dac}              & 92.91 & 90.71 & 86.30 & 74.84 & 73.55 & 66.92 & 57.17 & 32.16 & - & - & - & - \\
& SELF~\cite{nguyen2019self}                  & - & 91.13 & - & 63.59 & - & 66.71 & - & 35.56 & 92.76 & 89.07 & 70.53 & 53.83  \\
& ELR~\cite{liu2020early}                     & 92.12 & 91.43 & 88.87 & \textbf{80.69} & 74.68 & 68.43 & 60.05 & 30.27 & 93.28 & 90.35 & 74.20 & 73.73 \\
& Ours    & \textbf{94.14} & \textbf{92.64} & \textbf{89.23} & 78.58 & \textbf{75.77} & \textbf{71.38} & \textbf{62.69} & \textbf{38.72} & \textbf{94.07} & \textbf{90.45} & \textbf{77.80} & \textbf{74.08} \\
\midrule
\multirow{5}{*}{WRN28-10} & ERM + Early Stop & 87.86 & 83.40 & 76.92 & 63.54 & 68.46 & 55.43 & 40.78 & 20.25 & 90.61 & 79.93 & 69.85 & 50.30 \\
& MentorNet~\cite{jiang2018mentornet} & 92.0 & 89.0 & - & 49.0 & 73.0 & 68.0 & - & 35.0 & - & - & - & - \\
& DAC~\cite{thulasidasan2019dac} & 93.25 & 90.93 & 87.58 & 70.80 & 75.75 & 68.20 & 59.44 & 34.06 & - & - & - & - \\
& SELF~\cite{nguyen2019self} & - & \textbf{93.34} & - & 67.41 & - & 72.48 & - & 42.06 & - & - & - & - \\
& Ours & \textbf{94.84} & 93.23 & \textbf{89.42} & \textbf{80.13} & \textbf{77.71} & \textbf{72.60} & \textbf{64.87} & \textbf{44.17} & \textbf{95.11} & \textbf{91.07} & \textbf{80.59} & \textbf{75.30} \\
\bottomrule
\end{tabular}
}
\end{small}
\end{center}
\end{table*}

\section{Application I: Learning with Noisy Label}
Given the improved generalization of self-adaptive training over ERM under noise, we show the applications of our approach which outperforms the state-of-the-art with a significant gap.

\subsection{Problem formulation}
Given a set of noisy training data $\{(\x_i, \widetilde{\y}_i)\}_n \in \mathcal{\widetilde{D}}$, where $\mathcal{\widetilde{D}}$ is the distribution of noisy data and
$\widetilde{\y}_i$ is the noisy label for each uncorrupted sample $\x_i$, the goal is to be robust to the label noise in the training data and improve the classification performance on clean test data that are sampled from the clean distribution $\mathcal{D}$.

\subsection{Experiments on CIFAR datasets}
\label{sec:exp_label_noise}
\noindent\textbf{Setup}\quad
We consider the cases that, with different noise rates, the labels are (a) assigned uniformly at random (i.e., the symmetric noise), and (b) flipped according to the class-conditioned rules~\cite{patrini2017making} (i.e., the asymmetric noise): for CIFAR10, the label noise is generated by mapping Truck $\rightarrow$ Automobile, Bird $\rightarrow$ Airplane, Deer $\rightarrow$ Horse, Cat $\leftrightarrow$ Dog; for CIFAR100, the noise is generated by flipping each class to the next class circularly. Following previous work~\cite{zhang2018gce,thulasidasan2019dac}, we conduct the experiments on the CIFAR10 and CIFAR100 datasets~\cite{krizhevsky2009cifar} and use the ResNet-34~\cite{he2016deep} / Wide ResNet-28-10~\cite{zagoruyko2016wide} (WRN28-10) as our base classifier. The networks are implemented on PyTorch~\cite{paszke2019pytorch} and are optimized using SGD with an initial learning rate of 0.1, a momentum of 0.9, a weight decay of 0.0005, a batch size of 256, and the total training epochs of 200. The learning rate is decayed to zero using the cosine annealing schedule~\cite{loshchilov2016sgdr}. We use the standard data augmentation with random horizontal flipping and cropping. We report the average performance over 3 trials.

\begin{table}[t]
\caption{Influence of the two components of our approach.}
\label{tab:ablation_ema_reweight}
\begin{center}
\begin{small}
\setlength{\tabcolsep}{1.5mm}{
\begin{tabular}{lcccc}
\toprule
 & \multicolumn{2}{c}{CIFAR10} & \multicolumn{2}{c}{CIFAR100}\\
 \cmidrule(l{3pt}r{3pt}){2-3} \cmidrule(l{3pt}r{3pt}){4-5}
Label Noise Rate              & 0.4   & 0.8   & 0.4   & 0.8   \\
\midrule
Ours                    & \textbf{92.64} & \textbf{78.58} & \textbf{71.38} & \textbf{38.72} \\
- Re-weighting          & 92.49 & 78.10 & 69.52 & 36.78 \\
-  Exponential Moving Average        & 72.00 & 28.17 & 50.93 & 11.57 \\
\bottomrule
\end{tabular}
}
\end{small}
\end{center}
\end{table}

\medskip
\noindent\textbf{Main results}\quad
We summarize the experiments in Table~\ref{tab:noisy_cls_all}. Most of the results are directly cited from original papers with the same experiment settings; the results of Label Smoothing~\cite{szegedy2016rethinking},  Mixup~\cite{zhang2017mixup}, Joint Opt~\cite{tanaka2018joint}, and SCE~\cite{wang2019symmetric} are reproduced by rerunning the official open-sourced implementations.
From the table, we can see that our approach outperforms the state-of-the-art methods in most entries by 1\% $\sim$ 5\% on both CIFAR10 and CIFAR100 datasets, using different backbones. We observe that the improvements by our approach are consistent under both symmetric and asymmetric noise, demonstrating the robustness of self-adaptive training against various kinds of noise. Notably, unlike the Joint Opt, DAC, and SELF methods that require multiple iterations of training, our method enjoys the same computational budget as ERM.

\begin{table}
\caption{Parameters sensitivity when label noise of 40\% is injected into the CIFAR10 training set.}
\label{tab:ablation_alpha_es}
\begin{center}
\begin{small}
\setlength{\tabcolsep}{1.5mm}{
\begin{tabular}{lccccc}
\toprule
$\alpha$ & 0.6 & 0.8 & 0.9 & 0.95 & 0.99 \\
\midrule
Fix $\mathrm{E}_s = 60$ & 90.17 & 91.91 & \textbf{92.64} & 92.54 & 84.38 \\
\midrule\midrule
$\mathrm{E}_s$ & 20 & 40 & 60 & 80 & 100 \\
\midrule
Fix $\alpha = 0.9$ & 89.58 & 91.89 & \textbf{92.64} & 92.26 & 88.83 \\
\bottomrule
\end{tabular}
}
\end{small}
\end{center}
\end{table}

\medskip
\noindent\textbf{Ablation study and hyper-parameter sensitivity}\quad
First, we report the performance of ERM equipped with the simple early stopping scheme in the first row of Table~\ref{tab:noisy_cls_all}. We observe that our approach achieves substantial improvements over this baseline. This demonstrates that simply early stopping the training process is a sub-optimal solution.
Then, we further report the influences of the two individual components of our approach: Exponential Moving Average (EMA) and sample re-weighting scheme. As displayed in Table~\ref{tab:ablation_ema_reweight}, removing any component considerably hurts the performance under all noise rates, and removing the EMA scheme leads to a significant performance drop. This suggests that properly incorporating model predictions is important in our approach.
Finally, we analyze the sensitivity of our approach to the parameters $\alpha$ and $\mathrm{E}_s$ in Table~\ref{tab:ablation_alpha_es}, where we fix one parameter while varying the other. The performance is stable for various choices of $\alpha$ and $\mathrm{E}_s$, indicating that our approach is not sensitive to hyperparameter tuning.

\subsection{Experiments on ImageNet dataset}
The work of~\cite{russakovsky2015imagenet} suggested that the ImageNet dataset \cite{deng2009imagenet} contains annotation errors on its own even after several rounds of cleaning. Therefore, in this subsection, we use ResNet-50/101~\cite{he2016deep} to evaluate self-adaptive training on the large-scale ImageNet under both the standard setup (i.e., using original labels) and the case that 40\% of the training labels are corrupted. We provide the experimental details in Appendix~\ref{sec:setup_imagenet} and report model performance on the ImageNet validation set in terms of top1 accuracy in Table~\ref{tab:imagenet}. We can see that self-adaptive training consistently improves the ERM baseline by a considerable margin under all settings using different models. Specifically, the improvement can be as large as 3\% in absolute for the larger ResNet-101 when 40\% of the training labels are corrupted. The results validate the effectiveness of our approach on the large-scale dataset and larger model.

\begin{figure}[t]
    \centering
    \includegraphics[width=.45\textwidth]{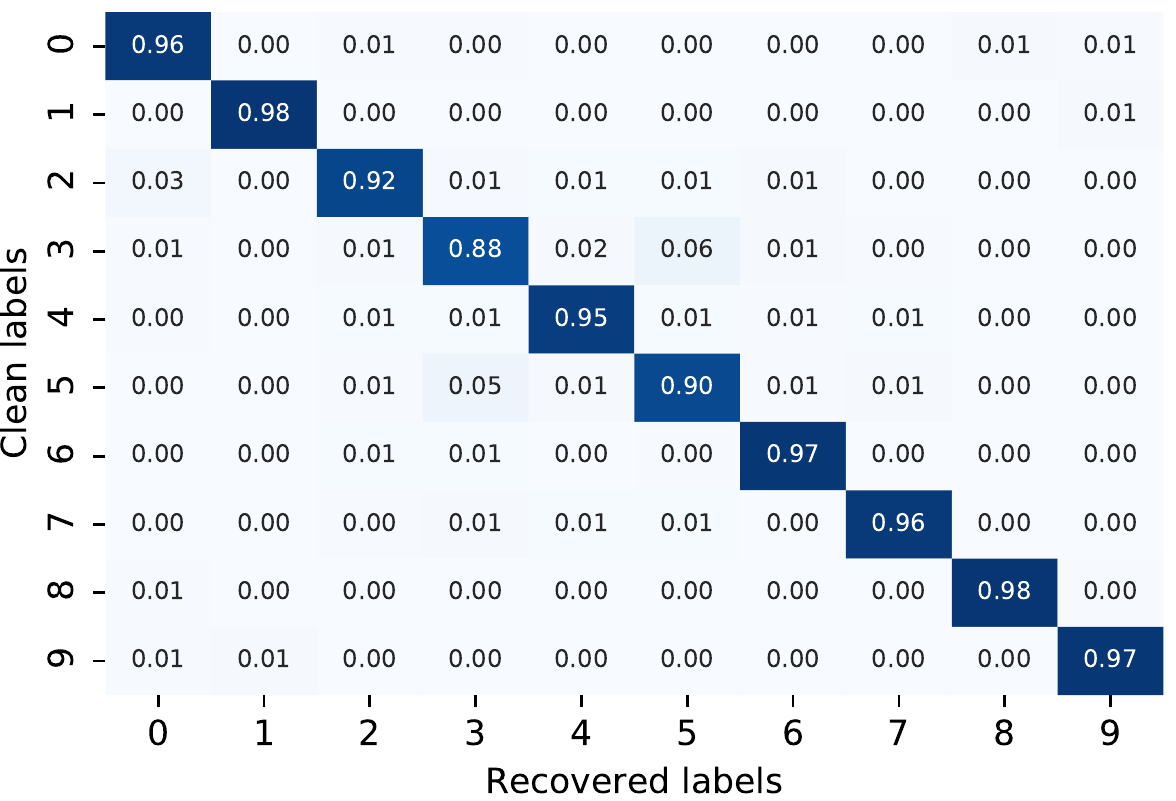}
    \caption{Confusion matrix of recovered labels w.r.t clean labels on CIFAR10 training set with 40\% of label noise. The overall recovery accuracy is 94.65\%.
    }
    \label{fig:conf_mat}
\end{figure}

\begin{table}[t]
\caption{Top1 Accuracy (\%) on ImageNet validation set.}
\label{tab:imagenet}
\begin{center}
\begin{small}
\begin{tabular}{lcccc}
\toprule
 & \multicolumn{2}{c}{ResNet-50} & \multicolumn{2}{c}{ResNet-101} \\
 \cmidrule(l{3pt}r{3pt}){2-3} \cmidrule(l{3pt}r{3pt}){4-5}
Label Noise Rate & 0.0 & 0.4 & 0.0 & 0.4 \\
\midrule
ERM     & 76.8 & 69.5 & 78.2 & 70.2 \\
Ours    & \textbf{77.2} & \textbf{71.5} & \textbf{78.7} & \textbf{73.5}  \\
\bottomrule
\end{tabular}
\end{small}
\end{center}
\end{table}

\subsection{Label recovery of self-adaptive training}
\label{sec:label_rec}
We demonstrate that our approach is able to recover the true labels from noisy training labels: we obtain the recovered labels by the moving average targets $\bm{t}_i$ and compute the recovered accuracy as $\frac{1}{n}\sum_i \mathbbm{1}\{ \argmax~\y_i = \argmax~\bm{t}_i \}$, where $\y_i$ is the clean label of each training sample. When 40\% of labels are corrupted in the CIFAR10 and ImageNet training set, our approach successfully corrects a huge amount of labels and obtains recovered accuracy of 94.6\% and 81.1\%, respectively. We also display the confusion matrix of recovered labels w.r.t the clean labels on CIFAR10 in Fig.~\ref{fig:conf_mat}, where we can see that our approach performs well for all classes.

\subsection{Investigation of sample weights}
We further inspect the re-weighting scheme of self-adaptive training. Following the procedure in Section~\ref{sec:label_rec}, we display the average sample weights in Fig.~\ref{fig:weights}. In the figure, the $(i, j)$th block contains the average weight of samples with clean label $i$ and recovered label $j$, the white areas represent the case that no sample lies in the cell. We see that the weights on the diagonal blocks are clearly higher than those on non-diagonal blocks. The figure indicates that, aside from its impressive ability to recover the correct labels, self-adaptive training could properly down-weight the noisy examples.

\begin{figure}[t]
\centering
\includegraphics[width=.45\textwidth]{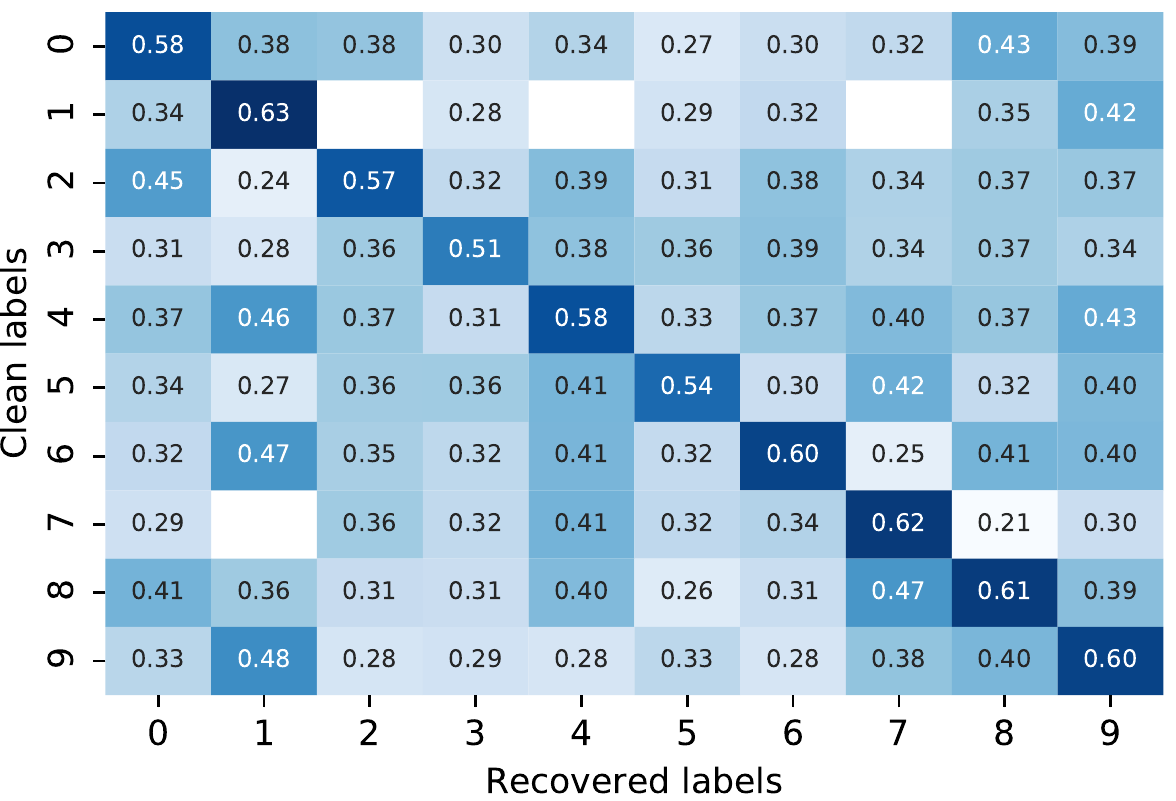}
\caption{Average sample weights $w_i$ under various labels. The white areas represent the case that no sample lies in the cell.}
\label{fig:weights}
\end{figure}

\begin{table*}[t]
\caption{Selective classification error rate (\%) on CIFAR10, SVHN, and Dogs vs. Cats datasets
for various coverage rates~(\%). Mean and standard deviation are calculated over 3 trials.
The best entries and those overlap with them are marked bold.}
\label{tab:sel_cls}
\begin{center}
\begin{small}
\begin{tabular}{lcccccc}
\toprule
Dataset & Coverage& Ours & Deep Gamblers~\cite{ziyindeepgambler} & SelectiveNet~\cite{geifman2019selectivenet} & SR~\cite{geifman2017selective} & MC-dropout~\cite{geifman2017selective} \\
\midrule
\multirow{7}{*}{CIFAR10} &  100 & 6.05$\pm$0.20 & 6.12$\pm$0.09 & 6.79$\pm$0.03 & 6.79$\pm$0.03 & 6.79$\pm$0.03  \\
 &  95 & \textbf{3.37$\pm$0.05} & \textbf{3.49$\pm$0.15} & 4.16$\pm$0.09 & 4.55$\pm$0.07 & 4.58$\pm$0.05  \\
 &  90 & \textbf{1.93$\pm$0.09} & 2.19$\pm$0.12 & 2.43$\pm$0.08 & 2.89$\pm$0.03 & 2.92$\pm$0.01  \\
 &  85 & \textbf{1.15$\pm$0.18} & \textbf{1.09$\pm$0.15} & 1.43$\pm$0.08 & 1.78$\pm$0.09 & 1.82$\pm$0.09  \\
 &  80 & \textbf{0.67$\pm$0.10} & \textbf{0.66$\pm$0.11} & 0.86$\pm$0.06 & 1.05$\pm$0.07 & 1.08$\pm$0.05  \\
 &  75 & \textbf{0.44$\pm$0.03} & 0.52$\pm$0.03 & \textbf{0.48$\pm$0.02} & 0.63$\pm$0.04 & 0.66$\pm$0.05  \\
 &  70 & \textbf{0.34$\pm$0.06} & 0.43$\pm$0.07 & \textbf{0.32$\pm$0.01} & 0.42$\pm$0.06 & 0.43$\pm$0.05 \\
\midrule
\multirow{5}{*}{SVHN}   & 100 & 2.75$\pm$0.09 & 3.24$\pm$0.09 & 3.21$\pm$0.08 & 3.21$\pm$0.08 & 3.21$\pm$0.08  \\
 &  95  & \textbf{0.96$\pm$0.09} & 1.36$\pm$0.02 & 1.40$\pm$0.01 & 1.39$\pm$0.05 & 1.40$\pm$0.05  \\
 &  90  & \textbf{0.60$\pm$0.05} & 0.76$\pm$0.05 & 0.82$\pm$0.01 & 0.89$\pm$0.04 & 0.90$\pm$0.04  \\
 &  85  & \textbf{0.45$\pm$0.02} & 0.57$\pm$0.07 & 0.60$\pm$0.01 & 0.70$\pm$0.03 & 0.71$\pm$0.03  \\
 &  80  & \textbf{0.43$\pm$0.01} & 0.51$\pm$0.05 & 0.53$\pm$0.01 & 0.61$\pm$0.02 & 0.61$\pm$0.01  \\
\midrule
\multirow{5}{*}{Dogs vs. Cats}  & 100 & 3.01$\pm$0.17 & 2.93$\pm$0.17 & 3.58$\pm$0.04 & 3.58$\pm$0.04 & 3.58$\pm$0.04  \\
 &  95 & \textbf{1.25$\pm$0.05} & \textbf{1.23$\pm$0.12} & 1.62$\pm$0.05 & 1.91$\pm$0.08 & 1.92$\pm$0.06  \\
 &  90 & \textbf{0.59$\pm$0.04} & \textbf{0.59$\pm$0.13} & 0.93$\pm$0.01 & 1.10$\pm$0.08 & 1.10$\pm$0.05  \\
 &  85 & \textbf{0.25$\pm$0.11} & 0.47$\pm$0.10 & 0.56$\pm$0.02 & 0.82$\pm$0.06 & 0.78$\pm$0.06  \\
 &  80 & \textbf{0.15$\pm$0.06} & 0.46$\pm$0.08 & 0.35$\pm$0.09 & 0.68$\pm$0.05 & 0.55$\pm$0.02 \\
\bottomrule
\end{tabular}
\end{small}
\end{center}
\end{table*}

\section{Application II: Selective Classification}
\subsection{Problem formulation}
Selective classification, a.k.a. classification with rejection, trades classifier coverage for accuracy~\cite{el2010foundations}, where the coverage is defined as the fraction of classified samples in the dataset; the classifier is allowed to output ``don't know'' for certain samples. The task focuses on the noise-free setting and allows the classifier to abstain from potential out-of-distribution samples or samples that lies in the tail of data distribution, that is, making prediction only on the samples with confidence.

Formally, a selective classifier is a composition of two functions $(f_{\theta}, g)$, where $f_{\theta}$ is the conventional $c$-class classifier and $g$ is the selection function that reveals the underlying uncertainty of inputs.
Given an input $\x$, selective classifier outputs
\begin{equation}
\label{eq:selective_classifier}
    (f_{\theta}, g)(\x) = \begin{cases}
                \mathrm{Abstain}, & g(\x) > \tau; \\
                f_{\theta}(\x), & \mathrm{otherwise,}
                \end{cases}
\end{equation}
for a given threshold $\tau$ that controls the trade-off.

\subsection{Approach}
Inspired by~\cite{thulasidasan2019dac,ziyindeepgambler},
we adapt our presented approach in Algorithm~\ref{alg:sat_meta_alg} to the selective classification task.
We introduce an extra ($c+1$)th class (represents \emph{abstention}) during training
and replace selection function $g(\cdot)$ in 
Equation~\eqref{eq:selective_classifier} by $f_{\theta}(\cdot)_{c}$.
In this way, we can train a selective classifier in an end-to-end fashion.
Besides, unlike previous works that provide no explicit signal for learning abstention class,
we use model predictions as a guideline when designing the learning process.

Given  a mini-batch of data pairs $\{(\x_i, \y_i)\}_m$,
the model prediction $\p_i$ and its exponential moving average $\bm{t}_i$ for each sample,
we optimize the classifier $f$ by minimizing:
\begin{equation}
\label{eq:abstain_loss}
  \mathcal{L}(f_{\theta})=-\frac{1}{m}\sum_i [\bm{t}_{i,y_i} \log \p_{i,y_i}+(1-\bm{t}_{i,y_i}) \log \p_{i,c}],
\end{equation}
where $y_i$ is the index of the non-zero element in the one hot label vector $\y_i$. The first term measures the cross entropy loss between prediction and original label $\y_i$, in order to learn a good multi-class classifier.
The second term acts as the selection function and identifies uncertain samples in datasets.
$\bm{t}_{i,y_i}$ dynamically trades-off these two terms:
if $\bm{t}_{i,y_i}$ is very small, the sample is deemed as uncertain and the second term enforces the selective classifier to learn to abstain from this sample;
if $\bm{t}_{i,y_i}$ is close to 1, the loss recovers the standard cross entropy minimization
and enforces the selective classifier to make perfect predictions.

\subsection{Experiments}

\noindent\textbf{Setup}\quad
We conduct the experiments on three datasets: CIFAR10~\cite{krizhevsky2009cifar}, SVHN~\cite{netzer2011svhn}, and Dogs vs. Cats~\cite{catsdogs}.
We compare our method with previous state-of-the-art methods on selective classification, including Deep Gamblers~\cite{ziyindeepgambler}, SelectiveNet~\cite{geifman2019selectivenet}, Softmax Response (SR), and MC-dropout~\cite{geifman2017selective}.
The experiments are based on the official open-sourced implementation\footnote{\url{https://github.com/Z-T-WANG/NIPS2019DeepGamblers}} of Deep Gamblers to ensure a fair comparison. We use the VGG-16 network~\cite{simonyan2014very} with batch normalization~\cite{ioffe2015batch} and dropout~\cite{srivastava2014dropout} as the base classifier in all experiments. The network is optimized using SGD with an initial learning rate of 0.1, a momentum of 0.9, a weight decay of 0.0005, a batch size of 128, and a total training epoch of 300. The learning rate is decayed by 0.5 in every 25 epochs. For our method, we set the hyper-parameters $\mathrm{E}_s=0, \alpha=0.99$.

\medskip
\noindent\textbf{Main results}\quad
The results of prior methods are cited from original papers and are summarized in Table~\ref{tab:sel_cls}. We see that our method achieves up to 50\% relative improvements compared with all other methods under various coverage rates, on all datasets. Notably, Deep Gamblers also introduces an extra abstention class in their method but without applying model predictions. The improvement of our method comes from the use of model predictions in the training process.

\section{Improved Representations Learning}
\label{sec:sat_ssl}

\subsection{Self-Supervised Self-Adaptive Training}
\label{sec:sat_ssl_instantiation}

\noindent\textbf{Instantiation}\quad
We consider training images $\{\x_i\}_n$ without label and use a deep network followed by a non-linear projection as encoder $f_{\theta}$. Then, we instantiate the meta Algorithm~\ref{alg:sat_meta_alg} for self-supervised learning as follows:
\begin{enumerate}
    \item \emph{Target initialization.} Since the labels are absent, each training target $\bm{t}_i$ is randomly and independently drawn from a standard normal distribution.
    
    \item \emph{Normalization function.} We directly normalize each representation by dividing its $\ell_2$ norm.
    
    \item \emph{Loss function.} The loss function is implemented as the Mean Square Error (MSE) between the normalized model prediction $\p_i$ and training target $\bm{t}_i$.
\end{enumerate}

The above instantiation of our meta algorithm, i.e., fitting models' own accumulated representations, suffices to learn decent representations, as exhibited by the blue curve in Fig.~\ref{fig:ssl_acc_at_training}. However, as discussed in Sec.~\ref{sec:sat_ssl_prelim} that the consistency of the training target also plays an essential role, especially in self-supervised representation learning, we introduce two components that further improve the representation learning with self-adaptive training.

\medskip
\noindent\textbf{Momentum encoder and predictor}\quad
We follow prior works~\cite{he2020momentum,grill2020bootstrap} to employ a momentum encoder $f_{\theta_m}$, whose parameters $\theta_m$ is also updated by the EMA scheme as
\begin{equation}
    \theta_m \leftarrow \beta\times \theta_m + (1 - \beta)\times\theta.
\end{equation}
With the slowly-evolving $f_{\theta_m}$, we can obtain a representation
\begin{equation}
    \bm{z}_i = f_{\theta_m}(\bm{x}^t_i),
\end{equation}
and construct the target $\bm{t}_i$ following the EMA scheme in Equation~\eqref{eq:meta_ema}.

Furthermore, to prevent the model from outputting the same representation for every image in each iteration (i.e., collapsing), we further use a predictor $g$ to transform the output of encoder $f_{\theta}$ to prediction 
\begin{equation}
    \bm{p}_i = g(f_{\theta}(\bm{x}_i)),
\end{equation}
where $g$ has the same number of output units as $f_{\theta}$.

\begin{figure}[t]
\centering
\includegraphics[width=.95\linewidth]{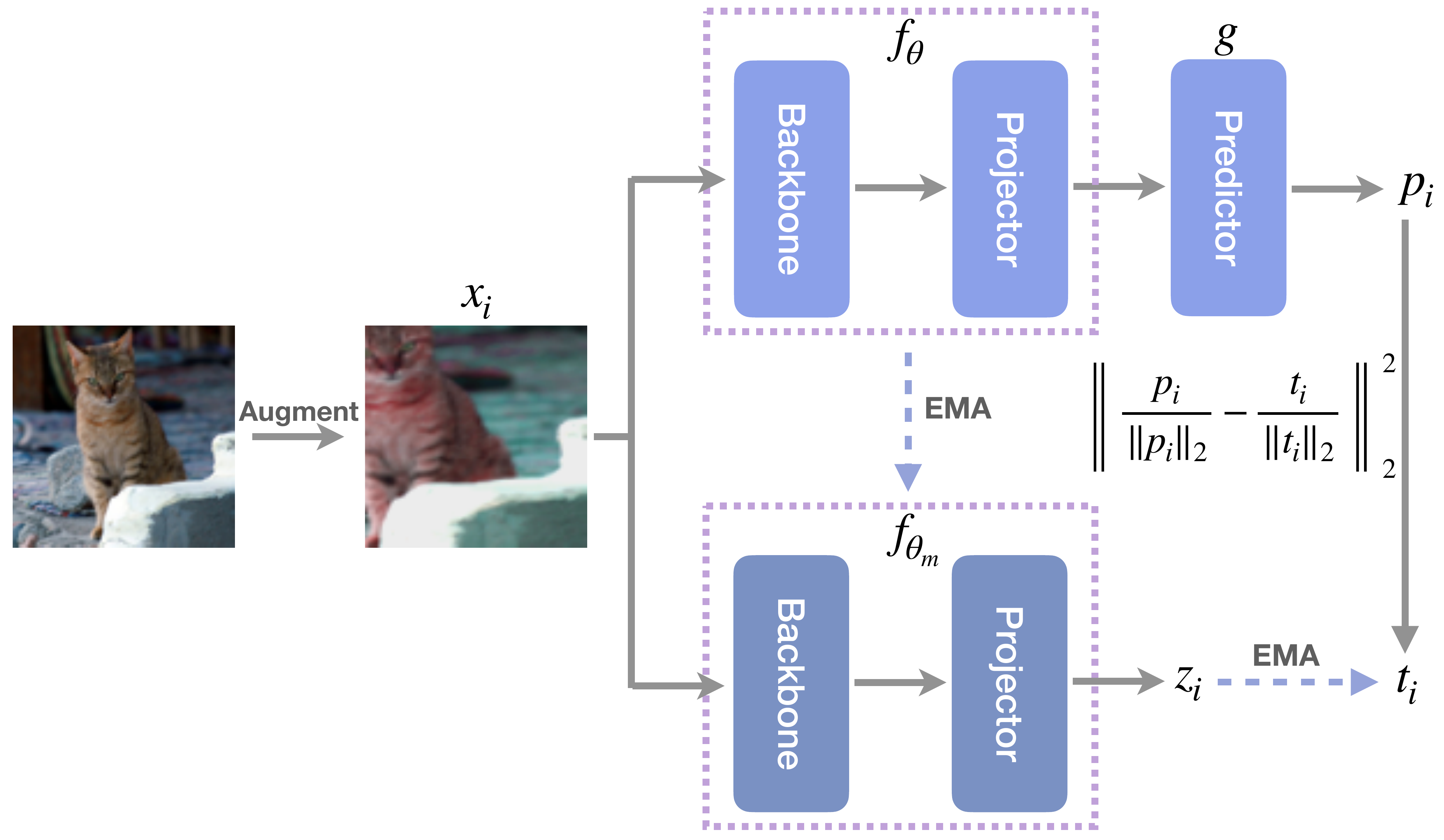}
\caption{Pipeline of Self-Supervised Self-Adaptive Training.}
\label{fig:ssl_pipeline}
\end{figure}

\medskip
\noindent\textbf{Putting everything together}\quad
We $\ell_2$-normalize $\bm{p}_i$ to $\bm{\widetilde{p}}_i = \bm{p}_i / \norm{\bm{p}_i}_2$ and $\bm{t}_i$ to $\bm{\widetilde{t}}_i = \bm{t}_i / \norm{\bm{t}_i}_2$. Finally, the MSE loss between the normalized predictions and accumulated representations
\begin{equation}
\label{eq:ssl_obj}
    \mathcal{L}(\p_i, \bm{t}_i; f_\theta, g) = \frac{1}{m}\sum_i \norm{\bm{\widetilde{p}}_{i} - \bm{\widetilde{t}}_{i}}^2_2
\end{equation}
is minimized to update the encoder $f_\theta$ and predictor $g$. We term this variant \emph{Self-Supervised Self-Adaptive Training}, and summarize the pseudo-code in Algorithm~\ref{alg:approach_ssl} and the overall pipeline in Fig.~\ref{fig:ssl_pipeline}. Our approach is straightforward to implement in practice and requires only single-view training, which significantly alleviates the heavy computation burden of data augmentation operations.

\begin{algorithm}[h]
\small
\caption{Self-Supervised Self-Adaptive Training}
\label{alg:approach_ssl}
\begin{algorithmic}[1]
\REQUIRE Data $\{\bm{x}_i\}_n$; encoder $f_{\theta}$ and momentum encoder $f_{\theta_m}$; predictor $g$; momentum terms $\alpha = 0.7$, $\beta = 0.99$
\STATE Initialize $\theta_m \leftarrow \theta$
\STATE Randomly initialize training target $\{\bm{t}_{i}\}_n \sim \mathcal{N}(\bm{0}, \mathbf{I})$
\REPEAT
\STATE Fetch augmented mini-batch data $\{\bm{x}^t_i\}_m$
\FOR{$i=1$ {\bfseries to} $m$  (in parallel)}
\STATE $\bm{p}_i = g(f_\theta(\bm{x}^t_i))$
\STATE $\bm{z}_i = f_{\theta_m}(\bm{x}^t_i)$
\STATE $\bm{\widetilde{z}}_i = \bm{z}_i / \norm{\bm{z}_i}_2$; \quad$\bm{\widetilde{t}}_i = \bm{t}_i / \norm{\bm{t}_i}_2$
\STATE $\bm{t}_{i} \leftarrow \alpha\times \bm{\widetilde{t}}_{i} + (1 - \alpha)\times \bm{\widetilde{z}}_i$
\ENDFOR
\STATE $\mathcal{L}(\p_i, \bm{t}_i; f_\theta, g) = \frac{1}{m}\sum_i \norm{\frac{\bm{p}_i}{\norm{\bm{p}_i}_2} - \frac{\bm{t}_i}{\norm{\bm{t}_i}_2}}^2_2$
\STATE Update $f_\theta$ and $g$ by SGD on $\mathcal{L}(\p_i, \bm{t}_i; f_\theta, g)$
\STATE Update $\theta_m \leftarrow \beta\times \theta_m + (1 - \beta)\times \theta$
\UNTIL{end of training}
\end{algorithmic}
\end{algorithm}

\medskip
\noindent\textbf{Methodology differences with prior works}\quad
BYOL~\cite{grill2020bootstrap} formulated the self-supervised training of deep models as predicting the representation of one augmented view of an image from the other augmented view of the same image. The self-supervised self-adaptive training shares some similarities with BYOL since both methods do not need to contrast with negative examples to prevent the collapsing issue. Instead of directly using the outputs of a momentum encoder as training targets, our approach uses the \emph{accumulated predictions} as training targets, which contain all historical view information for each image. As a result, our approach requires only a single view during training, which is much more efficient as shown in Sec.~\ref{sec:ssl_multi_view} and Fig.~\ref{fig:batch_time}.
Besides, NAT~\cite{bojanowski2017unsupervised} used an online clustering algorithm to assign noise as the training target for each image. Unlike NAT that fixes the noise while updating the noise assignment during training, our approach uses the noise as the \emph{initial} training target and updates the noise by model predictions in the subsequent training process. InstDisc~\cite{wu2018unsupervised} used a memory bank to store the representation of each image in order to construct positive and negative samples for the contrastive objective. By contrast, our method gets rid of the negative samples and only matches the prediction with the training target of the same image (i.e., positive sample).

\subsection{Bypassing collapsing issues}
\label{fig:sat_collapse}
We note that there exist trivial local minima when we directly optimize the MSE loss between predictions and training targets due to the absence of negative pairs: the encoder $f_{\theta}$ can simply output a constant feature vector for every data point to minimize the training loss, a.k.a. collapsing issue~\cite{grill2020bootstrap}.
Despite the existence of a collapsing solution, self-adaptive training intrinsically prevents the collapsing. The initial targets $\{\bm{t}_i\}_n$ are different for different classes (under the supervised setting) or even for different images (under the self-supervised setting), enforcing the models to learn different representations for different classes/images. Our empirical studies in Fig.~\ref{fig:acc_curve}, \ref{fig:ssl_acc_at_training} of the main body and Fig.~\ref{fig:acc_curve_r08} of the Appendix strongly support that deep neural networks are able to learn meaningful information from corrupted data or even random noise, and bypass the model collapse. Based on this, our approach significantly improves the supervised learning (see Fig.~\ref{fig:ce_acc_curve}) and self-supervised learning (see the blue curve in Fig.~\ref{fig:ssl_acc_at_training}) of deep neural networks.

\begin{figure}[t]
\centering
\includegraphics[width=.9\linewidth]{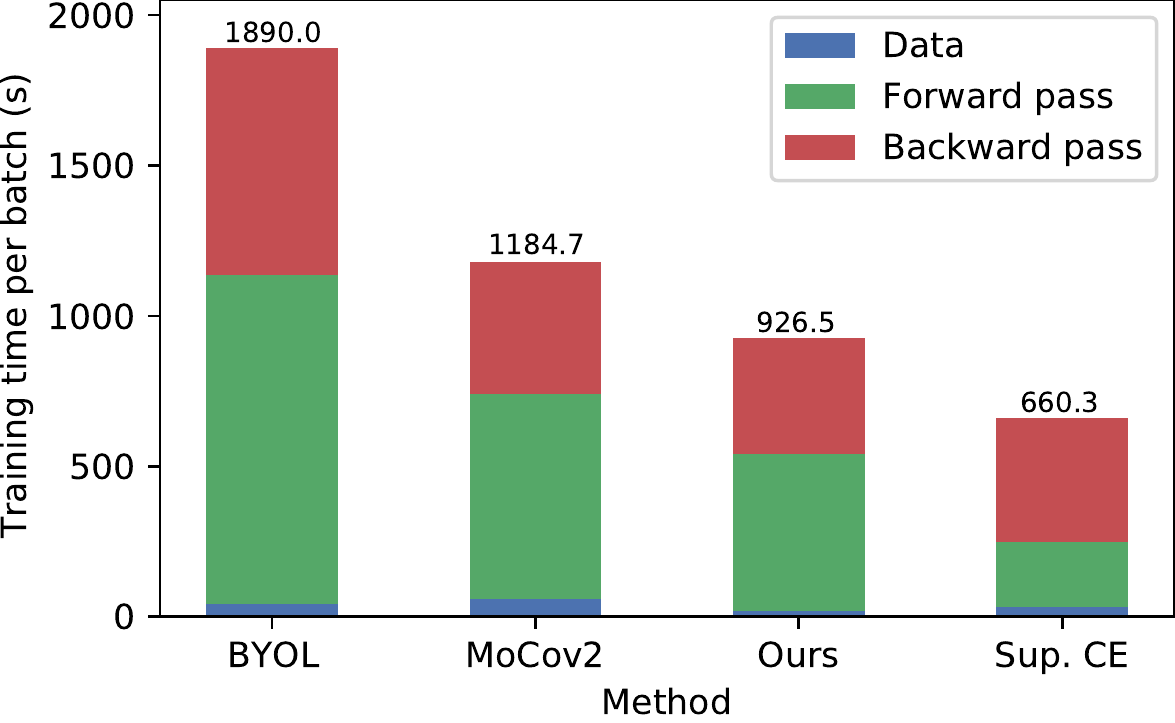}
\caption{Training time comparison on the ImageNet dataset in terms of the data preprocessing/forward pass/backward pass time. The comparison is performed on a machine with 8 V100 GPUs. It shows the efficiency of our method. }
\label{fig:batch_time}
\end{figure}

\subsection{Is multi-view training indispensable?}
\label{sec:ssl_multi_view}
The success of state-of-the-art self-supervised learning approaches~\cite{he2020momentum,grill2020bootstrap} largely hinges on the multi-view training scheme: they essentially use strong data augmentation operations to create multiple views (crops) of the same image and then match the representation of one view with the other views of the same image. Despite their promising performance, these methods suffer heavily from the computational burden of data pre-processing and the training on extra views.
Concretely, as shown in Fig.~\ref{fig:batch_time}, prior methods BYOL~\cite{grill2020bootstrap} and MoCo~\cite{he2020momentum} incur doubled training time compared with standard supervised cross entropy training. In contrast, since our method requires only single-view training, its training is only slightly slower than the supervised method and is much faster than MoCo and BYOL. See Appendix~\ref{sec:setup_ssl_training_time} for the detailed setup of this comparison.

We further conduct experiments to evaluate the performance of the multi-view training scheme on the CIFAR10 and STL10 datasets, which cast doubt on its necessity for learning a good representation. 
Concretely, we first pre-train a ResNet-18 encoder using different algorithms and then train a linear classifier on top of the encoder to evaluate the classification performance
(see Sec.~\ref{sec:linear_cls_setup} for details).
As shown in Table~\ref{tab:ssl_ablation_view}, although the performance of MoCo and BYOL is nearly halved on both datasets when using the single-view training scheme, the self-supervised adaptive training achieves comparable results under both settings. Moreover, our approach with single-view training even slightly outperforms the MoCo and BYOL with multi-view training. We attribute this superiority to the guiding principle of our algorithm: by dynamically incorporating the model predictions, each training target contains the relevant information about all the historical views of each image and, therefore, implicitly forces the model learning representations that are invariant to the historical views.

\subsection{On the power of momentum encoder and predictor}
\label{sec:ssl_menc_pred}
Momentum encoder and predictor are two important components of our approach and BYOL. The work of BYOL~\cite{grill2020bootstrap} showed that the algorithm may converge to a trivial solution if one of them is removed, not to mention removing both of them. As shown in Table~\ref{tab:ssl_ablation_compo}, however, our results challenge their conclusion: 1) with the predictor, the linear evaluation accuracy of either BYOL (>~85\%) or our method (>~90\%) is non-trivial, regardless of the absence and the configuration of momentum encoder; 2) without the predictor, the momentum encoder with a sufficiently large momentum can also improve the performance and bypass the collapsing issue. The results suggest that although both the predictor and the momentum encoder are indeed crucial for the performance of representation learning, either one of them with a proper configuration (i.e., the momentum term) suffices to avoid the collapse. We note that the latest version of \cite{grill2020bootstrap} also found that the momentum encoder can be removed without collapse when carefully tuning the learning rate of the predictor. Our results, however, are obtained using the same training setting.

Moreover, we find that self-supervised self-adaptive training exhibits impressive resistance to collapsing, despite using only single-view training. Our approach can learn a decent representation even when the predictor and momentum encoder are both removed (see the seventh row of Table~\ref{tab:ssl_ablation_compo}). We hypothesize that the resistance comes from the consistency of the training target due to our EMA scheme. This hypothesis is also partly supported by the observation that the learning of BYOL heavily depends on the slowly-evolving momentum encoder.

\begin{table}[t]
\caption{Necessity of the multi-view training scheme in self-supervised learning. Top1 test Accuracy (\%) on CIFAR10 and STL10 datasets using the ResNet-18 backbone.
}
\label{tab:ssl_ablation_view}
\begin{center}
\begin{small}
\setlength{\tabcolsep}{1.5mm}{
\begin{tabular}{lccc}
\toprule
Method & \# Views & CIFAR10 & STL10\\
\midrule
\multirow{2}{*}{MoCov2~\cite{chen2020improved}} & 1 & 46.36 & 39.65 \\
 & 2 & 91.68 & 88.90 \\
\midrule
\multirow{2}{*}{BYOL~\cite{grill2020bootstrap}} & 1 & 44.36 & 42.60 \\
 & 2 & 91.78 & 89.60 \\
\midrule
\multirow{2}{*}{Ours} & 1 & \textbf{92.27} & 89.55 \\
 & 2 & 91.56 & \textbf{90.05} \\
\bottomrule
\end{tabular}
}
\end{small}
\end{center}
\end{table}

\begin{table}[t]
\caption{Influence of the momentum encoder and predictor. The performance is measured by the top1 test Accuracy (\%) on the CIFAR10 dataset using the ResNet-18 backbone. The entries with $^\dag$ indicate that $\alpha = 0.9$.}
\label{tab:ssl_ablation_compo}
\begin{center}
\begin{small}
\setlength{\tabcolsep}{1.5mm}{
\begin{tabular}{lcccc}
\toprule
Method & Predictor & Momentum $\beta$ & Accuracy \\
\midrule

\multirow{6}{*}{BYOL~\cite{grill2020bootstrap}} & \multirow{3}{*}{$\times$} & 0.0 & 25.64 \\
 &  & 0.99 & 26.87 \\
 &  & 0.999 & 72.44 \\
 \cmidrule{2-4}
 & \multirow{3}{*}{$\checkmark$} & 0.0 & 85.22 \\
 &  & 0.99 & 91.78 \\
 & & 0.999 & 90.68 \\
\midrule
\multirow{6}{*}{Ours} & \multirow{3}{*}{$\times$} & 0.0 $^\dag$ & 78.36 \\
 &  & 0.99 $^\dag$ & 79.68 \\
 &  & 0.999 $^\dag$ & 83.58 \\
 \cmidrule{2-4}
 & \multirow{3}{*}{$\checkmark$} & 0.0 & 90.18 \\
 &  & 0.99 & \textbf{92.27} \\
 &  & 0.999 & 90.92 \\
\bottomrule
\end{tabular}
}
\end{small}
\end{center}
\end{table}

\section{Application III: Linear Evaluation Protocol in Self-Supervised Learning}
\label{sec:linear_cls}

\subsection{Experimental setup}
\label{sec:linear_cls_setup}

\noindent\textbf{Datasets and data augmentations}\quad
We conduct experiments on four benchmarks: CIFAR10/CIFAR100~\cite{krizhevsky2009cifar} with 50k images; STL10~\cite{coates2011analysis} with 105k images; ImageNet~\cite{deng2009imagenet} with $\sim$1.2m images. The choice of data augmentations follows prior works~\cite{he2020momentum,chen2020simple}: we take a random crop from each image and resize it to a fixed size (i.e., $32\times 32$ for CIFAR10/CIFAR100, $96\times 96$ for STL10, and $224\times 224$ for ImageNet);
the crop is then randomly transformed by color jittering, horizontal flip, and grayscale conversion (and Gaussian blurring for ImageNet).

\medskip
\noindent\textbf{Network architecture}\quad
The encoders $f_{\theta}$ and $f_{\theta_m}$ consist of a backbone of ResNet-18~\cite{he2016deep}/ResNet-50~\cite{he2016deep}/AlexNet~\cite{krizhevsky2012image} and a projector that is instantiated by a multi-layer perceptron (MLP). We use the output of the last global average pooling layer of the backbone as the extracted feature vectors. Following prior works~\cite{chen2020simple}, the output vectors of the backbone are transformed by the projector MLP to a dimension of 256. Besides, the predictor $g$ is also instantiated by an MLP with the same architecture as the projector in $f_{\theta}$. In our implementation, all the MLPs have one hidden layer of size 4,096, followed by a batch normalization~\cite{ioffe2015batch} layer and the ReLU activation~\cite{nair2010rectified}.

\begin{table}[t]
\caption{Linear evaluation protocol on CIFAR10, CIFAR100, and STL10 datasets using different backbones.}
\label{tab:ssl_sota}
\begin{center}
\begin{small}
\setlength{\tabcolsep}{1.5mm}{
\begin{tabular}{llccc}
\toprule
Backbone & Method & CIFAR10 & CIFAR100 & STL10\\
\midrule 
\multirow{6}{*}{AlexNet} 
 & SplitBrain~\cite{zhang2017split} & 67.1 & 39.0 & - \\
 & DeepCluster~\cite{caron2018deep} & 77.9 & 41.9 & - \\
 & InstDisc~\cite{wu2018unsupervised} & 70.1 & 39.4 & - \\
 & AND~\cite{huang2019unsupervised} & 77.6 & 47.9 & - \\
 & SeLa~\cite{asano2020self} & 83.4 & 57.4 & - \\
 & CMC~\cite{tian2019contrastive} & - & - & 83.28 \\
 & Ours & \textbf{83.55} & \textbf{59.80} & \textbf{83.75} \\
\midrule
\multirow{4}{*}{ResNet-50} & MoCov2~\cite{chen2020improved} & 93.20 & 69.48 & 91.95 \\
 & SimCLR\cite{chen2020simple} & 93.08 & 67.92 & 90.90 \\
 & BYOL~\cite{grill2020bootstrap} & 93.48 & 68.48 & 92.40 \\
 & Ours & \textbf{94.04} & \textbf{70.16} & \textbf{92.60} \\
\bottomrule
\end{tabular}
}
\end{small}
\end{center}
\end{table}

\medskip
\noindent\textbf{Self-supervised pre-training settings}\quad
On the CIFAR10, CIFAR100, and STL10 datasets, we optimize the networks using an SGD optimizer with a momentum of 0.9 and a weight decay of 0.00005. We use a batch size of 512 for all methods in all experiments and train the networks for 800 epochs using 4 NVIDIA GTX 1080Ti GPUs. The hyper-parameters of our method are set to $\alpha=0.7$, $\beta = 0.99$.
On the ImageNet dataset, we use a batch size of 1024, a LARS~\cite{you2017large} optimizer with a weight decay of 0.000001, and train the networks for 200 epochs on 8 GPUs.
The base learning rate is set to 2.0 and is scaled linearly with respect to the base batch size 256 following~\cite{goyal2017accurate}. For our method, we set the initial values of the momentum terms $\alpha$/$\beta$ to 0.5/0.99, which are then increased to 0.8/1.0 following the prior work~\cite{grill2020bootstrap}.
During training, the learning rate is warmed up for the first 30 epochs (10 epochs when using the ImageNet dataset) and is then adjusted according to the cosine annealing schedule~\cite{loshchilov2016sgdr}.

\medskip
\noindent\textbf{Linear evaluation protocol}\quad
Following the common practice~\cite{he2020momentum,chen2020simple,grill2020bootstrap}, we evaluate the representation learned by self-supervised pre-training using linear classification protocol. That is, we remove the projector in $f_{\theta}$ and the predictor $g$, fix the parameters of the backbone of the encoder $f_{\theta}$, and train a supervised linear classifier on top of the features extracted from the encoder. On the ImageNet dataset, the linear classifier is trained for 100 epochs with a weight decay of 0.0, a batch size of 512, and an initial learning rate of 2.0 that is decayed to 0 according to the cosine annealing schedule~\cite{loshchilov2016sgdr}.
On the rest datasets, the linear classifier is trained for 100 epochs with a weight decay of 0.0 and a batch size of 512. The initial learning rate is set to 0.4 and is decayed by a factor of 0.1 at the 60th and 80th training epochs. The performance is measured in terms of the top1 test accuracy of the linear classifier.

\subsection{Comparison with the state-of-the-art}
We firstly conduct experiments on the CIFAR10/100~\cite{krizhevsky2009cifar} and STL10~\cite{coates2011analysis} datasets and compare our self-supervised self-adaptive training with three state-of-the-art methods, including the contrastive learning methods MoCo~\cite{he2020momentum}, SimCLR~\cite{chen2020simple}, and bootstrap method BYOL~\cite{grill2020bootstrap}. For fair comparisons, we use the same code base and the same experimental settings for all the methods, following their official open-sourced implementations. We carefully adjust the hyper-parameters on the CIFAR10 dataset for each method and use the same parameters on the rest datasets. Besides, we also conduct experiments using AlexNet as the backbone and compare the performance of our method with the reported results of prior methods.  The results are summarized in Table~\ref{tab:ssl_sota}. We can see that, despite using only single-view training, our self-supervised self-adaptive training consistently obtains better performance than all other methods on all datasets with different backbones.

\begin{table}[t]
\caption{Linear evaluation protocol on ImageNet dataset using the ResNet-50~\cite{he2016deep} backbone. The entries marked with $\dag$ are cited from~\cite{chen2020exploring}; those marked by $\ddag$ are cited from~\cite{caron2020unsupervised}.}
\label{tab:ssl_sota_in1k}
\begin{center}
\begin{small}
\setlength{\tabcolsep}{1.5mm}{
\begin{tabular}{lccc}
\toprule
Method & \# Views & Epochs & Linear Acc. (\%) \\
\midrule 
InstDisc~\cite{wu2018unsupervised} & 2 & 200 & 56.5 \\
MoCo~\cite{he2020momentum} & 2 & 200 &  60.6 \\
CPCv2~\cite{henaff2020data} & 2 & 200 & 63.8 \\
SimCLR~\cite{chen2020simple} & 2 & 200 &  66.6 \\
MoCov2~\cite{chen2020improved} & 2 & 200 &  67.5 \\
BYOL~\cite{grill2020bootstrap}$^\dag$ & 2 & 200 & 70.6 \\
SwAV~\cite{caron2020unsupervised}$^\dag$ & 2 & 200 & 69.1 \\
SimSiam~\cite{chen2020exploring}$^\dag$ & 2 & 200 & 70.0 \\
Ours & 1 & 200 & 67.4 \\
Ours & 2 & 200 & \textbf{72.8} \\
\midrule
CMC~\cite{tian2019contrastive} & 2 & 240 & 60.0 \\
BYOL~\cite{grill2020bootstrap} & 2 & 300 & 72.4 \\
BarlowTwins~\cite{zbontar2021barlow} & 2 & 300 & 71.4 \\
MoCov3~\cite{chen2021empirical} & 2 & 300 & 72.8 \\
SeLa~\cite{asano2020self} & 2 & 400 & 61.5 \\
SeLav2~\cite{asano2020self}$^\ddag$ & 2 & 400 & 67.2 \\
DeepClusterv2~\cite{caron2018deep}$^\ddag$ & 2 & 400 & 70.2 \\
SwAV~\cite{caron2020unsupervised}$^\ddag$ & 2 & 400 & 70.1 \\
PIRL~\cite{misra2020self} & 2 & 800 & 63.6 \\
SimCLRv2~\cite{chen2020big} & 2 & 800 & 71.7 \\
\bottomrule
\end{tabular}
}
\end{small}
\end{center}
\end{table}

\begin{figure*}[t]
\begin{subfigure}{.49\linewidth}
    \centering
    \includegraphics[width=.83\linewidth]{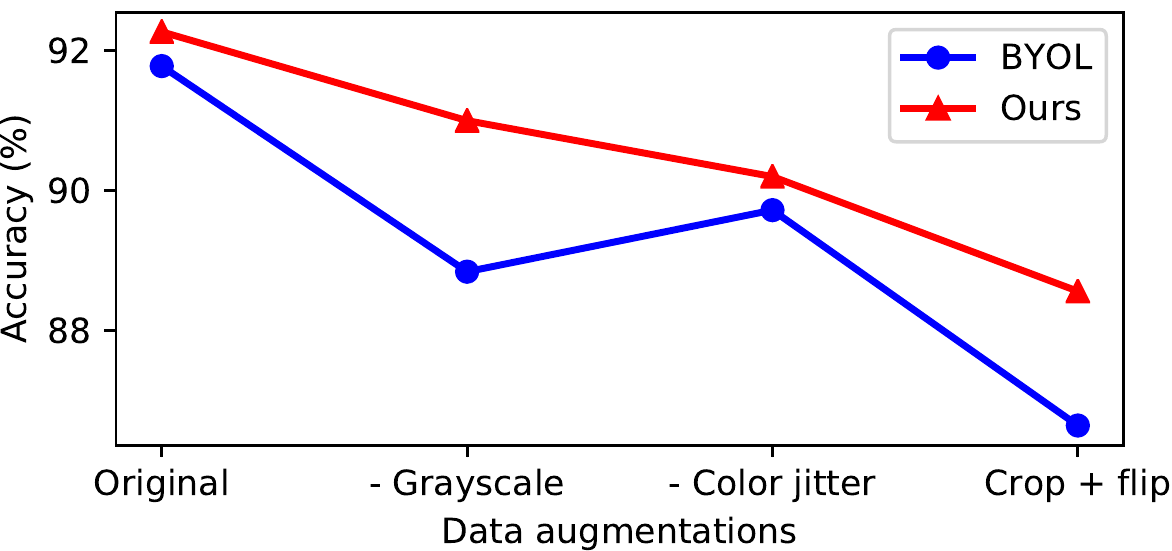}
    \caption{Sensitivity to data augmentation.}
    \label{fig:ssl_ablation_aug}
\end{subfigure}
\begin{subfigure}{.49\linewidth}
    \centering
    \includegraphics[width=.8\linewidth]{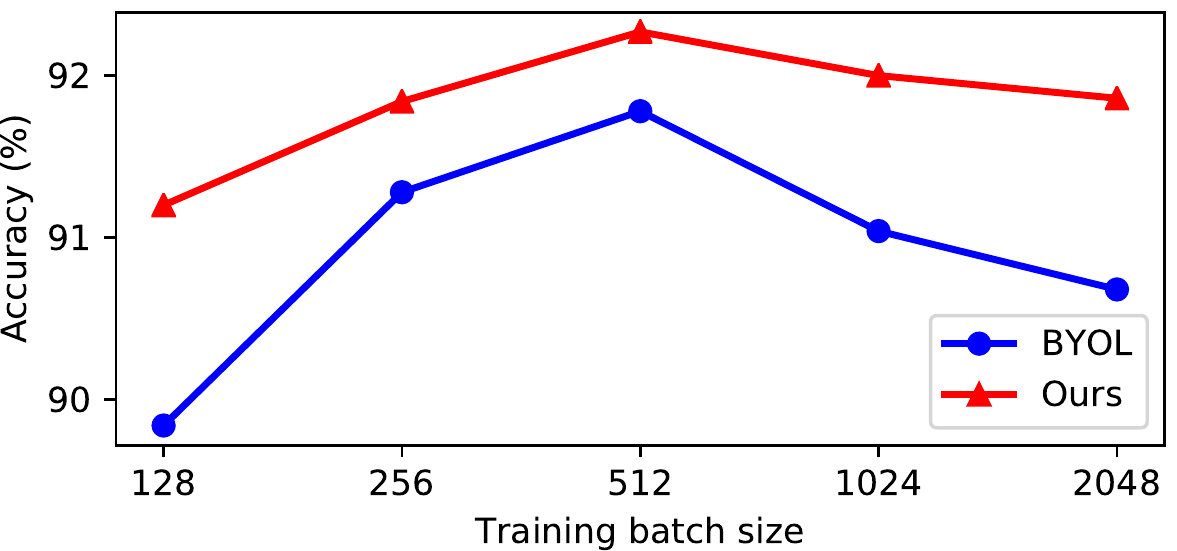}
    \caption{Sensitivity to training batch size.}
    \label{fig:ssl_ablation_bs}
\end{subfigure}
\caption{Sensitivity of our method to data augmentations and training batch size. The performance is measured by the top1 test Accuracy (\%) on the CIFAR10 dataset using the ResNet-18 backbone.}
\label{fig:ssl_ablation_aug_bs}
\end{figure*}

Moreover, we follow the standard setups in self-supervised learning on the large-scale ImageNet~\cite{russakovsky2015imagenet} dataset (i.e., using ResNet-50 backbone, 200 pre-training epochs, and the linear evaluation protocol), and report the comparisons with the state-of-the-art methods in Table~\ref{tab:ssl_sota_in1k}. We find that, while using 50\% fewer augmented views during training, the single-view self-adaptive training achieves 67.4\% top1 accuracy, which is comparable with the two-view SimCLR and MoCov2 and slightly worse than BYOL. We conjecture that, on the large-scale dataset, the training targets are updated less frequently than those on a smaller dataset, which results in performance degradation. Regarding this, to make a fair comparison, we also pre-train a ResNet-50 model based on the two-view variant of SAT, at a similar training cost as BYOL. From the results in Table~\ref{tab:ssl_sota_in1k}. We can observe that the two-view variant of SAT (a) outperforms all other methods under the same setting by a large margin, and (b) even performs on par with BYOL and MoCov3 trained by 300 epochs while outperforming the other methods with $>2\times$ training epochs.

\begin{table}
  \caption{Transfer learning to Pascal-VOC object detection and semantic segmentation with models pre-trained on ImageNet-1K datasets for 200 epochs. All entries are based on the Faster R-CNN~\cite{ren2015faster} architecture with the ResNet-50 C4 backbone~\cite{wu2019detectron2} for detection, and are based on ResNet-50 FCN~\cite{long2015fully} for segmentation. $\dag$: results cited from~\cite{chen2020exploring}. $*$: reproduction based on the publicly available checkpoint.}
  \label{tab:exp_voc}
  \centering
  \small
  \setlength{\tabcolsep}{1.5mm}{
  \begin{tabular}{@{}lccccc@{}}
    \toprule
    \multirow{2}{*}{Method} & \multirow{2}{*}{\# Views} & \multicolumn{3}{c}{VOC 07+12 Det.} & VOC 12 Seg. \\
     \cmidrule(l{3pt}r{3pt}){3-5} \cmidrule(l{3pt}r{3pt}){6-6}
     & & AP & AP$_{50}$ & AP$_{75}$ & mIoU \\
    \midrule
    Supervised$^\dag$ & 1 & 53.5 & 81.3 & 58.8 & 67.7$^*$ \\
    \midrule
    SimCLR~\cite{chen2020simple}$^\dag$ & 2 & 55.5 & 81.8 & 61.4 & - \\
    MoCov2~\cite{chen2020improved} & 2 & 57.0 & 82.4 & 63.6 & 66.7$^*$ \\
    SwAV~\cite{caron2020unsupervised}$^\dag$ & 2 & 55.4 & 81.5 & 61.4 & - \\
    BYOL~\cite{grill2020bootstrap}$^\dag$ & 2 & 55.3 & 81.4 & 61.1 & - \\
    SimSiam~\cite{chen2020exploring}$^\dag$ & 2 & 56.4 & 82.0 & 62.8 & - \\
    \midrule
    Ours & 1 & 55.2 & 81.5 & 60.9 & 64.8 \\
    Ours & 2 & 56.4 & 82.5 & 63.0 & 68.2 \\
    \bottomrule
  \end{tabular}
  }
\end{table}

Finally, in addition to the above experiments, we also evaluate the pre-trained models on two downstream tasks, object detection and semantic segmentation on the Pascal VOC dataset~\cite{everingham15}. We follow exactly the same settings as in prior works~\cite{he2020momentum,chen2020exploring} for fair comparisons, where all models are pre-trained for 200 epochs on ImageNet. 
We use the Faster R-CNN~\cite{ren2015faster} architecture with the ResNet-50 C4 backbone~\cite{wu2019detectron2} for detection and the ResNet-50 dilated FCN~\cite{long2015fully} architecture for segmentation. The experiments are conducted on the opensource codebases $\operatorname{detectron2}$~\cite{wu2019detectron2}/$\operatorname{mmsegmentation}$~\cite{mmseg2020} for the detection/segmentation tasks.
According to the results in Table~\ref{tab:exp_voc}, self-adaptive training performs on par with or even better than state-of-the-art self-supervised learning methods on both transfer learning tasks. Moreover, we can also observe that SAT consistently achieves superior performance to the supervised baseline.

\begin{table}[t]
\caption{Parameters sensitivity on the CIFAR10 dataset.}
\label{tab:ssl_ablation_param}
\begin{center}
\begin{small}
\setlength{\tabcolsep}{1.5mm}{
\begin{tabular}{lccccc}
\toprule
$\alpha$ & 0.5 & 0.6 & 0.7 & 0.8 & 0.9 \\
\midrule
Fix $\beta = 0.99$ & 91.68 & 91.76 & \textbf{92.27} & 92.08 & 91.68 \\
\midrule\midrule
$\beta$ & 0.9 & 0.95 & 0.99 & 0.995 & 0.999 \\
\midrule
Fix $\alpha = 0.7$ & 90.68 & 91.60 & \textbf{92.27} & 92.16 & 90.92 \\
\bottomrule
\end{tabular}
}
\end{small}
\end{center}
\end{table}

\subsection{Sensitivity of self-supervised self-adaptive training}
\noindent\textbf{Sensitivity to hyper-parameters}\quad
We study how the two momentum parameters $\alpha$ and $\beta$ affect the performance of our approach and report the results in Table~\ref{tab:ssl_ablation_param}. By varying one of the parameters while fixing the other, we observe that self-supervised self-adaptive training performs consistently well, which suggests that our approach is not sensitive to the choice of hyper-parameters.

\medskip
\noindent\textbf{Sensitivity to data augmentation}\quad
Data augmentation is one of the most essential ingredients of recent self-supervised learning methods: a large body of these methods, including ours, formulate the training objective as learning representations that encode the shared information across different views generated by data augmentation. As a result, prior methods, like MoCo and BYOL, fail in the single-view training setting. Self-supervised self-adaptive training, on the other hand, maintains a training target for each image, which contains all historical view information of this image. Therefore, we conjecture that our method should be more robust to data augmentations. To validate our conjecture, we evaluate our method and BYOL under the settings that some of the augmentation operators are removed. The results are shown in Fig.~\ref{fig:ssl_ablation_aug}. Removing any augmentation operators hurts the performance of both methods while our method is less affected. Specifically, when all augmentation operators except random crop and flip are removed, the performance of BYOL drops to 86.6\% while our method still obtains 88.6\% accuracy.

\medskip
\noindent\textbf{Sensitivity to training batch size}\quad
Recent contrastive learning methods require large batch size training (e.g., 1024 or even larger) for optimal performance, due to the need of comparing with massive negative samples. BYOL does not use negative samples and suggests that this issue can be mitigated. Here, since our method also gets rid of the negative samples, we make direct comparisons with BYOL at different batch sizes to evaluate the sensitivity of our method to batch size. For each method, the base learning rate is linearly scaled according to the batch size while the rest settings are kept unchanged. The results are shown in Fig.~\ref{fig:ssl_ablation_bs}. We can see that our method exhibits a smaller performance drop than BYOL at various batch sizes. Concretely, the accuracy of BYOL drops by 2\% at batch size 128 while that of ours drops by only 1\%.

\section{Related Works}

\noindent\textbf{Generalization of deep networks}\quad
Previous work~\cite{zhang2016understanding} systematically analyzed the capability of deep networks to overfit random noise. Their results show that traditional wisdom fails to explain the generalization of deep networks.
Another line of works~\cite{opper1995statistical,opper2001learning,advani2017high,spigler2018jamming,belkin2018reconciling,geiger2019jamming,nakkiran2019deep} observed an intriguing double-descent risk curve from the bias-variance trade-off. \cite{belkin2018reconciling,nakkiran2019deep} claimed that this observation challenges the conventional U-shaped risk curve in the textbook. Our work shows that this observation may stem from overfitting to noise; the phenomenon vanishes by a properly designed training process such as self-adaptive training.
To improve the generalization of deep networks, \cite{szegedy2016rethinking,pereyra2017regularizing} proposed label smoothing regularization that uniformly distributes $\epsilon$ of labeling weight to all classes and uses this soft label for training; \cite{zhang2017mixup} introduced mixup augmentation that extends the training distribution by dynamic interpolations between random paired input images and the associated targets during training.
This line of research is similar to ours as both methods use soft labels in the training. However, self-adaptive training is able to recover true labels from noisy labels and is more robust to noise.

\medskip
\noindent\textbf{Robust learning from corrupted data}\quad
Aside from the preprocessing-training approaches that have been discussed in the last paragraph of Section~\ref{sec:sat_supervised_inst}, there have also been many other works on learning from noisy data. To name a few, \cite{arpit2017closer,li2019gradient} showed that deep neural networks tend to fit clean samples first and overfitting to noise occurs in the later stage of training. \cite{li2019gradient} further proved that early stopping can mitigate the issues that are caused by label noise. \cite{reed2014training,dong2019distillation} incorporated model predictions into training by simple interpolation of labels and model predictions. We demonstrate that our exponential moving average and sample re-weighting schemes enjoy superior performance. Other works~\cite{zhang2018gce,wang2019symmetric} proposed alternative loss functions to cross entropy that are robust to label noise. They are orthogonal to ours and are ready to cooperate with our approach as shown in Appendix~\ref{sec:sat_sce}.
Beyond the corrupted data setting, recent works~\cite{furlanello2018born,xie2020self} propose a self-training scheme that also uses model predictions as training targets. However, they suffer from the heavy cost of multiple iterations of training, which is avoided by our approach. Temporal Ensembling~\cite{laine2017temporal} incorporated the 'ensemble' predictions as pseudo-labels for training. Different from ours, Temporal Ensembling focuses on the semi-supervised learning setting and only accumulates predictions for unlabeled data.

\medskip
\noindent\textbf{Self-supervised learning}\quad
Aiming to learn powerful representations, most self-supervised learning approaches typically first solve a proxy task without human supervision. For example, prior works proposed recovering input using auto-encoder~\cite{vincent2008extracting,pathak2016context}, generating pixels in the input space~\cite{kingma2013auto,goodfellow2014generative}, predicting rotation~\cite{gidaris2018unsupervised} and solving jigsaw~\cite{noroozi2016unsupervised}. Recently, contrastive learning methods~\cite{wu2018unsupervised,oord2018representation,tian2019contrastive,he2020momentum,chen2020simple,li2020prototypical} significantly advanced self-supervised representation learning. These approaches essentially used strong data augmentation techniques to create multiple views (crops) of the same image and discriminated the representation of different views of the same images (i.e., positive samples) from the views of other images (i.e., negative samples). Bootstrap methods eliminated the discrimination of positive and negative data pairs: the works of~\cite{caron2018deep,asano2020self} alternatively performed clustering on the representations and then used the cluster assignments as classification targets to update the model; \cite{caron2020unsupervised} swapped the cluster assignments between the two views of the same image as training targets; \cite{grill2020bootstrap} simply predicted the representation of one view from the other view of the same image; \cite{bojanowski2017unsupervised} formulated the self-supervised training objective as predicting a set of predefined noise. Our work follows the path of bootstrap methods. Going further than them, self-adaptive training is a general training algorithm that bridges supervised and self-supervised learning paradigms. Our approach casts doubt on the necessity of the costly multi-view training and works well with the single-view training scheme.

\section{Conclusion}

In this paper, we explore the possibility of a unified framework to bridge the supervised and self-supervised learning of deep neural networks. We first analyze the training dynamic of deep networks under these two learning settings and observe that useful information from data is distilled to model predictions. The observation occurs broadly even in the presence of data corruptions and the absence of labels, which motivates us to propose Self-Adaptive Training---a general training algorithm that dynamically incorporates model predictions into the training process. We demonstrate that our approach improves the generalization of deep neural networks under various kinds of training data corruption and enhances the representation learning using accumulated model predictions. Finally, we present three applications of self-adaptive training on learning with noisy labels, selective classification, and linear evaluation protocol in self-supervised learning, where our approach significantly advances the state-of-the-art.

\ifCLASSOPTIONcompsoc
  \section*{Acknowledgments}
\else
  \section*{Acknowledgment}
\fi

Lang Huang and Chao Zhang were supported by the National Nature Science Foundation of China under Grant 62071013 and 61671027, and the National Key R\&D Program of China under Grant 2018AAA0100300. Hongyang Zhang was supported in part by an NSERC Discovery Grant.

\ifCLASSOPTIONcaptionsoff
  \newpage
\fi



%


\bibliographystyle{IEEEtran}
\bibliography{main}

\begin{thebibliography}{100}
\providecommand{\url}[1]{#1}
\csname url@samestyle\endcsname
\providecommand{\newblock}{\relax}
\providecommand{\bibinfo}[2]{#2}
\providecommand{\BIBentrySTDinterwordspacing}{\spaceskip=0pt\relax}
\providecommand{\BIBentryALTinterwordstretchfactor}{4}
\providecommand{\BIBentryALTinterwordspacing}{\spaceskip=\fontdimen2\font plus
\BIBentryALTinterwordstretchfactor\fontdimen3\font minus
  \fontdimen4\font\relax}
\providecommand{\BIBforeignlanguage}[2]{{%
\expandafter\ifx\csname l@#1\endcsname\relax
\typeout{** WARNING: IEEEtran.bst: No hyphenation pattern has been}%
\typeout{** loaded for the language `#1'. Using the pattern for}%
\typeout{** the default language instead.}%
\else
\language=\csname l@#1\endcsname
\fi
#2}}
\providecommand{\BIBdecl}{\relax}
\BIBdecl

\bibitem{huang2020self}
L.~Huang, C.~Zhang, and H.~Zhang, ``Self-adaptive training: beyond empirical
  risk minimization,'' in \emph{Advances in Neural Information Processing
  Systems}, vol.~33, 2020.

\bibitem{simonyan2014very}
K.~Simonyan and A.~Zisserman, ``Very deep convolutional networks for
  large-scale image recognition,'' in \emph{International Conference on
  Learning Representations}, 2015.

\bibitem{he2016deep}
K.~He, X.~Zhang, S.~Ren, and J.~Sun, ``Deep residual learning for image
  recognition,'' in \emph{Proceedings of the IEEE Conference on Computer Vision
  and Pattern Recognition}, 2016, pp. 770--778.

\bibitem{girshick2014rich}
R.~Girshick, J.~Donahue, T.~Darrell, and J.~Malik, ``Rich feature hierarchies
  for accurate object detection and semantic segmentation,'' in
  \emph{Proceedings of the IEEE Conference on Computer Vision and Pattern
  Recognition}, 2014, pp. 580--587.

\bibitem{long2015fully}
J.~Long, E.~Shelhamer, and T.~Darrell, ``Fully convolutional networks for
  semantic segmentation,'' in \emph{Proceedings of the IEEE Conference on
  Computer Vision and Pattern Recognition}, 2015, pp. 3431--3440.

\bibitem{radford2018improving}
A.~Radford, K.~Narasimhan, T.~Salimans, and I.~Sutskever, ``Improving language
  understanding by generative pre-training,'' 2018.

\bibitem{radford2019language}
A.~Radford, J.~Wu, R.~Child, D.~Luan, D.~Amodei, and I.~Sutskever, ``Language
  models are unsupervised multitask learners,'' \emph{OpenAI blog}, vol.~1,
  no.~8, p.~9, 2019.

\bibitem{brown2020language}
T.~B. Brown, B.~Mann, N.~Ryder, M.~Subbiah, J.~Kaplan, P.~Dhariwal,
  A.~Neelakantan, P.~Shyam, G.~Sastry, A.~Askell \emph{et~al.}, ``Language
  models are few-shot learners,'' in \emph{Advances in Neural Information
  Processing Systems}, vol.~33, 2020.

\bibitem{devlin2019bert}
J.~Devlin, M.-W. Chang, K.~Lee, and K.~Toutanova, ``Bert: Pre-training of deep
  bidirectional transformers for language understanding,'' in \emph{Proceedings
  of the Conference of the North American Chapter of the Association for
  Computational Linguistics: Human Language Technologies}, 2019, pp.
  4171--4186.

\bibitem{he2020momentum}
K.~He, H.~Fan, Y.~Wu, S.~Xie, and R.~Girshick, ``Momentum contrast for
  unsupervised visual representation learning,'' in \emph{Proceedings of the
  IEEE/CVF Conference on Computer Vision and Pattern Recognition}, 2020, pp.
  9729--9738.

\bibitem{chen2020simple}
T.~Chen, S.~Kornblith, M.~Norouzi, and G.~Hinton, ``A simple framework for
  contrastive learning of visual representations,'' in \emph{International
  Conference on Machine Learning}, 2020.

\bibitem{caron2020unsupervised}
M.~Caron, I.~Misra, J.~Mairal, P.~Goyal, P.~Bojanowski, and A.~Joulin,
  ``Unsupervised learning of visual features by contrasting cluster
  assignments,'' in \emph{Advances in Neural Information Processing Systems},
  vol.~33, 2020.

\bibitem{wu2018unsupervised}
Z.~Wu, Y.~Xiong, S.~X. Yu, and D.~Lin, ``Unsupervised feature learning via
  non-parametric instance discrimination,'' in \emph{Proceedings of the IEEE
  Conference on Computer Vision and Pattern Recognition}, 2018, pp. 3733--3742.

\bibitem{zhang2016understanding}
C.~Zhang, S.~Bengio, M.~Hardt, B.~Recht, and O.~Vinyals, ``Understanding deep
  learning requires rethinking generalization,'' in \emph{International
  Conference on Learning Representations}, 2017.

\bibitem{NIPS2019_9336}
V.~Nagarajan and J.~Z. Kolter, ``Uniform convergence may be unable to explain
  generalization in deep learning,'' in \emph{Advances in Neural Information
  Processing Systems}, 2019, pp. 11\,611--11\,622.

\bibitem{krizhevsky2009cifar}
A.~Krizhevsky and G.~E. Hinton, ``Learning multiple layers of features from
  tiny images,'' University of Toronto, Tech. Rep., 2009.

\bibitem{rolnick2017deep}
D.~Rolnick, A.~Veit, S.~Belongie, and N.~Shavit, ``Deep learning is robust to
  massive label noise,'' \emph{arXiv preprint arXiv:1705.10694}, 2017.

\bibitem{guan2018said}
M.~Y. Guan, V.~Gulshan, A.~M. Dai, and G.~E. Hinton, ``Who said what: Modeling
  individual labelers improves classification,'' in \emph{Thirty-Second AAAI
  Conference on Artificial Intelligence}, 2018.

\bibitem{li2019gradient}
M.~Li, M.~Soltanolkotabi, and S.~Oymak, ``Gradient descent with early stopping
  is provably robust to label noise for overparameterized neural networks,'' in
  \emph{International Conference on Artificial Intelligence and
  Statistics}.\hskip 1em plus 0.5em minus 0.4em\relax PMLR, 2020, pp.
  4313--4324.

\bibitem{deng2009imagenet}
J.~Deng, W.~Dong, R.~Socher, L.-J. Li, K.~Li, and L.~Fei-Fei, ``Imagenet: A
  large-scale hierarchical image database,'' in \emph{Proceedings of the IEEE
  Conference on Computer Vision and Pattern Recognition}.\hskip 1em plus 0.5em
  minus 0.4em\relax Ieee, 2009, pp. 248--255.

\bibitem{FairScale2021}
M.~Baines, S.~Bhosale, V.~Caggiano, N.~Goyal, S.~Goyal, M.~Ott, B.~Lefaudeux,
  V.~Liptchinsky, M.~Rabbat, S.~Sheiffer, A.~Sridhar, and M.~Xu, ``Fairscale: A
  general purpose modular pytorch library for high performance and large scale
  training,'' \url{https://github.com/facebookresearch/fairscale}, 2021.

\bibitem{brodley1996identifying}
C.~E. Brodley, M.~A. Friedl \emph{et~al.}, ``Identifying and eliminating
  mislabeled training instances,'' in \emph{Proceedings of the National
  Conference on Artificial Intelligence}, 1996, pp. 799--805.

\bibitem{brodley1999identifying}
C.~E. Brodley and M.~A. Friedl, ``Identifying mislabeled training data,''
  \emph{Journal of Artificial Intelligence Research}, vol.~11, pp. 131--167,
  1999.

\bibitem{zhu2003eliminating}
X.~Zhu, X.~Wu, and Q.~Chen, ``Eliminating class noise in large datasets,'' in
  \emph{Proceedings of the 20th International Conference on Machine Learning
  (ICML-03)}, 2003, pp. 920--927.

\bibitem{nguyen2019self}
D.~T. Nguyen, C.~K. Mummadi, T.~P.~N. Ngo, T.~H.~P. Nguyen, L.~Beggel, and
  T.~Brox, ``{SELF}: Learning to filter noisy labels with self-ensembling,'' in
  \emph{International Conference on Learning Representations}, 2020.

\bibitem{bagherinezhad2018label}
H.~Bagherinezhad, M.~Horton, M.~Rastegari, and A.~Farhadi, ``Label refinery:
  Improving imagenet classification through label progression,'' \emph{arXiv
  preprint arXiv:1805.02641}, 2018.

\bibitem{tanaka2018joint}
D.~Tanaka, D.~Ikami, T.~Yamasaki, and K.~Aizawa, ``Joint optimization framework
  for learning with noisy labels,'' in \emph{Proceedings of the IEEE Conference
  on Computer Vision and Pattern Recognition}, 2018, pp. 5552--5560.

\bibitem{teng1999correcting}
C.-M. Teng, ``Correcting noisy data.'' in \emph{International Conference on
  Machine Learning}.\hskip 1em plus 0.5em minus 0.4em\relax Citeseer, 1999, pp.
  239--248.

\bibitem{rousseeuw2005robust}
P.~J. Rousseeuw and A.~M. Leroy, \emph{Robust regression and outlier
  detection}.\hskip 1em plus 0.5em minus 0.4em\relax John Wiley \& Sons, 2005,
  vol. 589.

\bibitem{jiang2018mentornet}
L.~Jiang, Z.~Zhou, T.~Leung, L.-J. Li, and L.~Fei-Fei, ``Mentornet: Learning
  data-driven curriculum for very deep neural networks on corrupted labels,''
  in \emph{International Conference on Machine Learning}, 2018, pp. 2304--2313.

\bibitem{ren2018learning}
M.~Ren, W.~Zeng, B.~Yang, and R.~Urtasun, ``Learning to reweight examples for
  robust deep learning,'' in \emph{International Conference on Machine
  Learning}, 2018, pp. 4334--4343.

\bibitem{belkin2018reconciling}
M.~Belkin, D.~Hsu, S.~Ma, and S.~Mandal, ``Reconciling modern machine-learning
  practice and the classical bias--variance trade-off,'' \emph{Proceedings of
  the National Academy of Sciences}, vol. 116, no.~32, pp. 15\,849--15\,854,
  2019.

\bibitem{nakkiran2019deep}
P.~Nakkiran, G.~Kaplun, Y.~Bansal, T.~Yang, B.~Barak, and I.~Sutskever, ``Deep
  double descent: Where bigger models and more data hurt,'' in
  \emph{International Conference on Learning Representations}, 2020.

\bibitem{opper1995statistical}
M.~Opper, ``Statistical mechanics of learning: Generalization,'' \emph{The
  Handbook of Brain Theory and Neural Networks,}, pp. 922--925, 1995.

\bibitem{opper2001learning}
------, ``Learning to generalize,'' \emph{Frontiers of Life}, vol.~3, no. part
  2, pp. 763--775, 2001.

\bibitem{advani2017high}
M.~S. Advani, A.~M. Saxe, and H.~Sompolinsky, ``High-dimensional dynamics of
  generalization error in neural networks,'' \emph{Neural Networks}, vol. 132,
  pp. 428--446, 2020.

\bibitem{spigler2018jamming}
S.~Spigler, M.~Geiger, S.~d'Ascoli, L.~Sagun, G.~Biroli, and M.~Wyart, ``A
  jamming transition from under-to over-parametrization affects loss landscape
  and generalization,'' \emph{arXiv preprint arXiv:1810.09665}, 2018.

\bibitem{geiger2019jamming}
M.~Geiger, S.~Spigler, S.~d'Ascoli, L.~Sagun, M.~Baity-Jesi, G.~Biroli, and
  M.~Wyart, ``Jamming transition as a paradigm to understand the loss landscape
  of deep neural networks,'' \emph{Physical Review E}, vol. 100, no.~1, p.
  012115, 2019.

\bibitem{szegedy2013intriguing}
C.~Szegedy, W.~Zaremba, I.~Sutskever, J.~Bruna, D.~Erhan, I.~Goodfellow, and
  R.~Fergus, ``Intriguing properties of neural networks,'' in
  \emph{International Conference on Learning Representations}, 2014.

\bibitem{zhang2019theoretically}
H.~Zhang, Y.~Yu, J.~Jiao, E.~Xing, L.~El~Ghaoui, and M.~Jordan, ``Theoretically
  principled trade-off between robustness and accuracy,'' in
  \emph{International Conference on Machine Learning}, 2019, pp. 7472--7482.

\bibitem{croce2020reliable}
F.~Croce and M.~Hein, ``Reliable evaluation of adversarial robustness with an
  ensemble of diverse parameter-free attacks,'' in \emph{International
  Conference on Machine Learning}, 2020.

\bibitem{madry2017towards}
A.~Madry, A.~Makelov, L.~Schmidt, D.~Tsipras, and A.~Vladu, ``Towards deep
  learning models resistant to adversarial attacks,'' in \emph{International
  Conference on Learning Representations}, 2018.

\bibitem{szegedy2016rethinking}
C.~Szegedy, V.~Vanhoucke, S.~Ioffe, J.~Shlens, and Z.~Wojna, ``Rethinking the
  inception architecture for computer vision,'' in \emph{Proceedings of the
  IEEE Conference on Computer Vision and Pattern Recognition}, 2016, pp.
  2818--2826.

\bibitem{patrini2017making}
G.~Patrini, A.~Rozza, A.~Krishna~Menon, R.~Nock, and L.~Qu, ``Making deep
  neural networks robust to label noise: A loss correction approach,'' in
  \emph{Proceedings of the IEEE Conference on Computer Vision and Pattern
  Recognition}, 2017, pp. 1944--1952.

\bibitem{zhang2017mixup}
H.~Zhang, M.~Cisse, Y.~N. Dauphin, and D.~Lopez-Paz, ``mixup: Beyond empirical
  risk minimization,'' in \emph{International Conference on Learning
  Representations}, 2018.

\bibitem{zhang2018gce}
Z.~Zhang and M.~Sabuncu, ``Generalized cross entropy loss for training deep
  neural networks with noisy labels,'' in \emph{Advances in Neural Information
  Processing Systems}, 2018, pp. 8778--8788.

\bibitem{wang2019symmetric}
Y.~Wang, X.~Ma, Z.~Chen, Y.~Luo, J.~Yi, and J.~Bailey, ``Symmetric cross
  entropy for robust learning with noisy labels,'' in \emph{Proceedings of the
  IEEE International Conference on Computer Vision}, 2019, pp. 322--330.

\bibitem{thulasidasan2019dac}
S.~Thulasidasan, T.~Bhattacharya, J.~Bilmes, G.~Chennupati, and J.~Mohd-Yusof,
  ``Combating label noise in deep learning using abstention,'' in
  \emph{International Conference on Machine Learning}, 2019, pp. 6234--6243.

\bibitem{liu2020early}
S.~Liu, J.~Niles-Weed, N.~Razavian, and C.~Fernandez-Granda, ``Early-learning
  regularization prevents memorization of noisy labels,'' in \emph{Advances in
  Neural Information Processing Systems}, 2020.

\bibitem{zagoruyko2016wide}
S.~Zagoruyko and N.~Komodakis, ``Wide residual networks,'' in \emph{Proceedings
  of the British Machine Vision Conference}, 2016, pp. 87.1--87.12.

\bibitem{paszke2019pytorch}
A.~Paszke, S.~Gross, F.~Massa, A.~Lerer, J.~Bradbury, G.~Chanan, T.~Killeen,
  Z.~Lin, N.~Gimelshein, L.~Antiga \emph{et~al.}, ``Pytorch: An imperative
  style, high-performance deep learning library,'' in \emph{Advances in Neural
  Information Processing Systems}, 2019, pp. 8024--8035.

\bibitem{loshchilov2016sgdr}
I.~Loshchilov and F.~Hutter, ``{SGDR}: Stochastic gradient descent with warm
  restarts,'' in \emph{International Conference on Learning Representations},
  2017.

\bibitem{russakovsky2015imagenet}
O.~Russakovsky, J.~Deng, H.~Su, J.~Krause, S.~Satheesh, S.~Ma, Z.~Huang,
  A.~Karpathy, A.~Khosla, M.~Bernstein \emph{et~al.}, ``Imagenet large scale
  visual recognition challenge,'' \emph{International Journal of Computer
  Vision}, vol. 115, no.~3, pp. 211--252, 2015.

\bibitem{ziyindeepgambler}
Z.~Liu, Z.~Wang, P.~P. Liang, R.~Salakhutdinov, L.-P. Morency, and M.~Ueda,
  ``Deep gamblers: Learning to abstain with portfolio theory,'' in
  \emph{Advances in Neural Information Processing Systems}, 2019.

\bibitem{geifman2019selectivenet}
Y.~Geifman and R.~El-Yaniv, ``Selectivenet: A deep neural network with an
  integrated reject option,'' in \emph{International Conference on Machine
  Learning}, 2019, pp. 2151--2159.

\bibitem{geifman2017selective}
------, ``Selective classification for deep neural networks,'' in
  \emph{Advances in Neural Information Processing Systems}, 2017, pp.
  4878--4887.

\bibitem{el2010foundations}
R.~El-Yaniv and Y.~Wiener, ``On the foundations of noise-free selective
  classification,'' \emph{Journal of Machine Learning Research}, vol.~11, no.
  May, pp. 1605--1641, 2010.

\bibitem{netzer2011svhn}
Y.~Netzer, T.~Wang, A.~Coates, A.~Bissacco, B.~Wu, and A.~Y. Ng, ``Reading
  digits in natural images with unsupervised feature learning,'' 2011.

\bibitem{catsdogs}
``Dogs vs. cats dataset,'' \url{https://www.kaggle.com/c/dogs-vs-cats}.

\bibitem{ioffe2015batch}
S.~Ioffe and C.~Szegedy, ``Batch normalization: Accelerating deep network
  training by reducing internal covariate shift,'' in \emph{International
  Conference on Machine Learning}, 2015, pp. 448--456.

\bibitem{srivastava2014dropout}
N.~Srivastava, G.~Hinton, A.~Krizhevsky, I.~Sutskever, and R.~Salakhutdinov,
  ``Dropout: a simple way to prevent neural networks from overfitting,''
  \emph{The Journal of Machine Learning Research}, vol.~15, no.~1, pp.
  1929--1958, 2014.

\bibitem{grill2020bootstrap}
J.-B. Grill, F.~Strub, F.~Altch{\'e}, C.~Tallec, P.~Richemond, E.~Buchatskaya,
  C.~Doersch, B.~Avila~Pires, Z.~Guo, M.~Gheshlaghi~Azar \emph{et~al.},
  ``Bootstrap your own latent-a new approach to self-supervised learning,'' in
  \emph{Advances in Neural Information Processing Systems}, vol.~33, 2020.

\bibitem{bojanowski2017unsupervised}
P.~Bojanowski and A.~Joulin, ``Unsupervised learning by predicting noise,'' in
  \emph{International Conference on Machine Learning}, 2017, pp. 517--526.

\bibitem{chen2020improved}
X.~Chen, H.~Fan, R.~Girshick, and K.~He, ``Improved baselines with momentum
  contrastive learning,'' \emph{arXiv preprint arXiv:2003.04297}, 2020.

\bibitem{coates2011analysis}
A.~Coates, A.~Ng, and H.~Lee, ``An analysis of single-layer networks in
  unsupervised feature learning,'' in \emph{Proceedings of International
  Conference on Artificial Intelligence and Statistics}, 2011, pp. 215--223.

\bibitem{krizhevsky2012image}
A.~Krizhevsky, I.~Sutskever, and G.~E. Hinton, ``Imagenet classification with
  deep convolutional neural networks,'' in \emph{Advances in Neural Information
  Processing Systems}, vol.~25, 2012, pp. 1097--1105.

\bibitem{nair2010rectified}
V.~Nair and G.~E. Hinton, ``Rectified linear units improve restricted boltzmann
  machines,'' in \emph{International Conference on Machine Learning}, 2010.

\bibitem{zhang2017split}
R.~Zhang, P.~Isola, and A.~A. Efros, ``Split-brain autoencoders: Unsupervised
  learning by cross-channel prediction,'' in \emph{Proceedings of the IEEE
  Conference on Computer Vision and Pattern Recognition}, 2017, pp. 1058--1067.

\bibitem{caron2018deep}
M.~Caron, P.~Bojanowski, A.~Joulin, and M.~Douze, ``Deep clustering for
  unsupervised learning of visual features,'' in \emph{Proceedings of the
  European Conference on Computer Vision}, 2018, pp. 132--149.

\bibitem{huang2019unsupervised}
J.~Huang, Q.~Dong, S.~Gong, and X.~Zhu, ``Unsupervised deep learning by
  neighbourhood discovery,'' in \emph{International Conference on Machine
  Learning}, 2019, pp. 2849--2858.

\bibitem{asano2020self}
Y.~M. Asano, C.~Rupprecht, and A.~Vedaldi, ``Self-labelling via simultaneous
  clustering and representation learning,'' in \emph{International Conference
  on Learning Representations}, 2020.

\bibitem{tian2019contrastive}
Y.~Tian, D.~Krishnan, and P.~Isola, ``Contrastive multiview coding,'' in
  \emph{European Conference on Computer Vision}, 2020.

\bibitem{you2017large}
Y.~You, I.~Gitman, and B.~Ginsburg, ``Large batch training of convolutional
  networks,'' \emph{arXiv preprint arXiv:1708.03888}, 2017.

\bibitem{goyal2017accurate}
P.~Goyal, P.~Doll{\'a}r, R.~Girshick, P.~Noordhuis, L.~Wesolowski, A.~Kyrola,
  A.~Tulloch, Y.~Jia, and K.~He, ``Accurate, large minibatch sgd: Training
  imagenet in 1 hour,'' \emph{arXiv preprint arXiv:1706.02677}, 2017.

\bibitem{chen2020exploring}
X.~Chen and K.~He, ``Exploring simple siamese representation learning,'' in
  \emph{Proceedings of the IEEE/CVF Conference on Computer Vision and Pattern
  Recognition}, 2021, pp. 15\,750--15\,758.

\bibitem{henaff2020data}
O.~Henaff, ``Data-efficient image recognition with contrastive predictive
  coding,'' in \emph{International Conference on Machine Learning}.\hskip 1em
  plus 0.5em minus 0.4em\relax PMLR, 2020, pp. 4182--4192.

\bibitem{zbontar2021barlow}
J.~Zbontar, L.~Jing, I.~Misra, Y.~LeCun, and S.~Deny, ``Barlow twins:
  Self-supervised learning via redundancy reduction,'' in \emph{International
  Conference on Machine Learning}, 2021.

\bibitem{chen2021empirical}
X.~Chen, S.~Xie, and K.~He, ``An empirical study of training self-supervised
  vision transformers,'' in \emph{Proceedings of the IEEE International
  Conference on Computer Vision}, 2021.

\bibitem{misra2020self}
I.~Misra and L.~v.~d. Maaten, ``Self-supervised learning of pretext-invariant
  representations,'' in \emph{Proceedings of the IEEE/CVF Conference on
  Computer Vision and Pattern Recognition}, 2020, pp. 6707--6717.

\bibitem{chen2020big}
T.~Chen, S.~Kornblith, K.~Swersky, M.~Norouzi, and G.~E. Hinton, ``Big
  self-supervised models are strong semi-supervised learners,'' \emph{Advances
  in Neural Information Processing Systems}, vol.~33, pp. 22\,243--22\,255,
  2020.

\bibitem{ren2015faster}
S.~Ren, K.~He, R.~Girshick, and J.~Sun, ``Faster r-cnn: Towards real-time
  object detection with region proposal networks,'' \emph{Advances in Neural
  Information Processing Systems}, vol.~28, pp. 91--99, 2015.

\bibitem{wu2019detectron2}
Y.~Wu, A.~Kirillov, F.~Massa, W.-Y. Lo, and R.~Girshick, ``Detectron2,''
  \url{https://github.com/facebookresearch/detectron2}, 2019.

\bibitem{everingham15}
M.~Everingham, S.~M.~A. Eslami, L.~Van~Gool, C.~K.~I. Williams, J.~Winn, and
  A.~Zisserman, ``The pascal visual object classes challenge: A
  retrospective,'' \emph{International Journal of Computer Vision}, vol. 111,
  no.~1, pp. 98--136, Jan. 2015.

\bibitem{mmseg2020}
Contributors, ``{MMSegmentation}: Openmmlab semantic segmentation toolbox and
  benchmark,'' \url{https://github.com/open-mmlab/mmsegmentation}, 2020.

\bibitem{pereyra2017regularizing}
G.~Pereyra, G.~Tucker, J.~Chorowski, {\L}.~Kaiser, and G.~Hinton,
  ``Regularizing neural networks by penalizing confident output
  distributions,'' \emph{arXiv preprint arXiv:1701.06548}, 2017.

\bibitem{arpit2017closer}
D.~Arpit, S.~Jastrz{\k{e}}bski, N.~Ballas, D.~Krueger, E.~Bengio, M.~S. Kanwal,
  T.~Maharaj, A.~Fischer, A.~Courville, Y.~Bengio \emph{et~al.}, ``A closer
  look at memorization in deep networks,'' in \emph{International Conference on
  Machine Learning}.\hskip 1em plus 0.5em minus 0.4em\relax JMLR. org, 2017,
  pp. 233--242.

\bibitem{reed2014training}
S.~Reed, H.~Lee, D.~Anguelov, C.~Szegedy, D.~Erhan, and A.~Rabinovich,
  ``Training deep neural networks on noisy labels with bootstrapping,'' in
  \emph{Workshop at the International Conference on Learning Representation},
  2015.

\bibitem{dong2019distillation}
B.~Dong, J.~Hou, Y.~Lu, and Z.~Zhang, ``Distillation $\approx$ early stopping?
  harvesting dark knowledge utilizing anisotropic information retrieval for
  overparameterized neural network,'' \emph{arXiv preprint arXiv:1910.01255},
  2019.

\bibitem{furlanello2018born}
T.~Furlanello, Z.~Lipton, M.~Tschannen, L.~Itti, and A.~Anandkumar, ``Born
  again neural networks,'' in \emph{International Conference on Machine
  Learning}, 2018, pp. 1607--1616.

\bibitem{xie2020self}
Q.~Xie, M.-T. Luong, E.~Hovy, and Q.~V. Le, ``Self-training with noisy student
  improves imagenet classification,'' in \emph{Proceedings of the IEEE/CVF
  Conference on Computer Vision and Pattern Recognition}, 2020, pp.
  10\,687--10\,698.

\bibitem{laine2017temporal}
S.~Laine and T.~Aila, ``Temporal ensembling for semi-supervised learning,'' in
  \emph{International Conference on Learning Representations}, 2017.

\bibitem{vincent2008extracting}
P.~Vincent, H.~Larochelle, Y.~Bengio, and P.-A. Manzagol, ``Extracting and
  composing robust features with denoising autoencoders,'' in
  \emph{International Conference on Machine Learning}, 2008, pp. 1096--1103.

\bibitem{pathak2016context}
D.~Pathak, P.~Krahenbuhl, J.~Donahue, T.~Darrell, and A.~A. Efros, ``Context
  encoders: Feature learning by inpainting,'' in \emph{Proceedings of the IEEE
  Conference on Computer Vision and Pattern Recognition}, 2016, pp. 2536--2544.

\bibitem{kingma2013auto}
D.~P. Kingma and M.~Welling, ``Auto-encoding variational bayes,'' in
  \emph{International Conference on Learning Representation}, 2014.

\bibitem{goodfellow2014generative}
I.~Goodfellow, J.~Pouget-Abadie, M.~Mirza, B.~Xu, D.~Warde-Farley, S.~Ozair,
  A.~Courville, and Y.~Bengio, ``Generative adversarial nets,'' in
  \emph{Advances in Neural Information Processing Systems}, 2014, pp.
  2672--2680.

\bibitem{gidaris2018unsupervised}
S.~Gidaris, P.~Singh, and N.~Komodakis, ``Unsupervised representation learning
  by predicting image rotations,'' in \emph{International Conference on
  Learning Representation}, 2018.

\bibitem{noroozi2016unsupervised}
M.~Noroozi and P.~Favaro, ``Unsupervised learning of visual representations by
  solving jigsaw puzzles,'' in \emph{European Conference on Computer
  Vision}.\hskip 1em plus 0.5em minus 0.4em\relax Springer, 2016, pp. 69--84.

\bibitem{oord2018representation}
A.~v.~d. Oord, Y.~Li, and O.~Vinyals, ``Representation learning with
  contrastive predictive coding,'' \emph{arXiv preprint arXiv:1807.03748},
  2018.

\bibitem{li2020prototypical}
J.~Li, P.~Zhou, C.~Xiong, R.~Socher, and S.~C. Hoi, ``Prototypical contrastive
  learning of unsupervised representations,'' \emph{International Conference on
  Learning Representations}, 2021.

\bibitem{kingma2014adam}
D.~P. Kingma and J.~Ba, ``Adam: A method for stochastic optimization,'' in
  \emph{International Conference on Learning Representations}, 2015.

\bibitem{hendrycks2018benchmarking}
D.~Hendrycks and T.~Dietterich, ``Benchmarking neural network robustness to
  common corruptions and perturbations,'' in \emph{International Conference on
  Learning Representations}, 2018.

\end{thebibliography}

%

\begin{IEEEbiography}[{\includegraphics[width=1in,height=1.25in,clip,keepaspectratio]{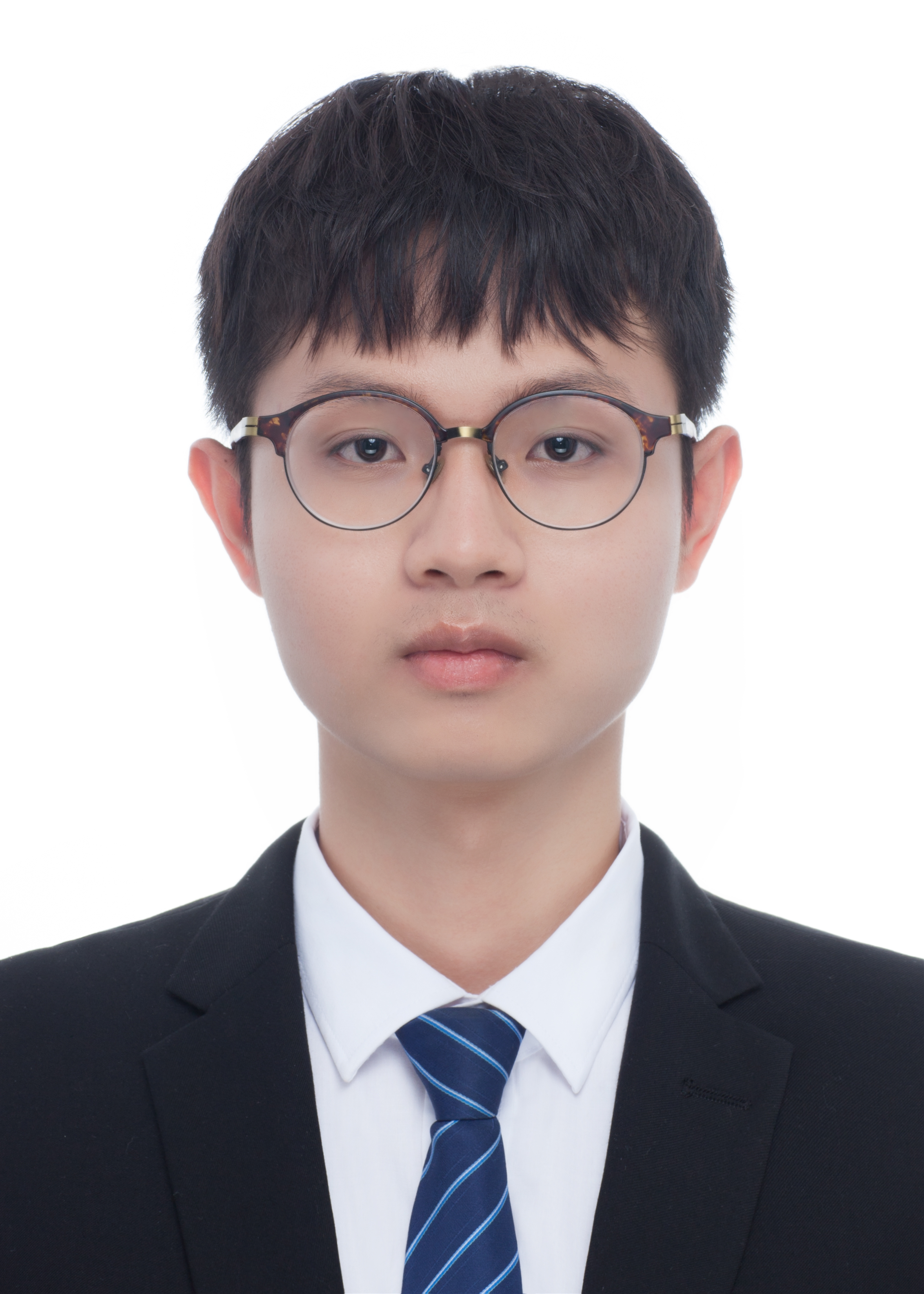}}]{Lang Huang}
received the Master's degree from the Department of Machine Intelligence, School of Electronics Engineering and Computer Science, Peking University in 2021. He is currently a Ph.D. student at the Department of Information \& Communication Engineering, The University of Tokyo. His research interests include self-supervised representation learning, learning from noisy data, and their applications.
\end{IEEEbiography}

\begin{IEEEbiography}[{\includegraphics[width=1in,height=1.25in,clip,keepaspectratio]{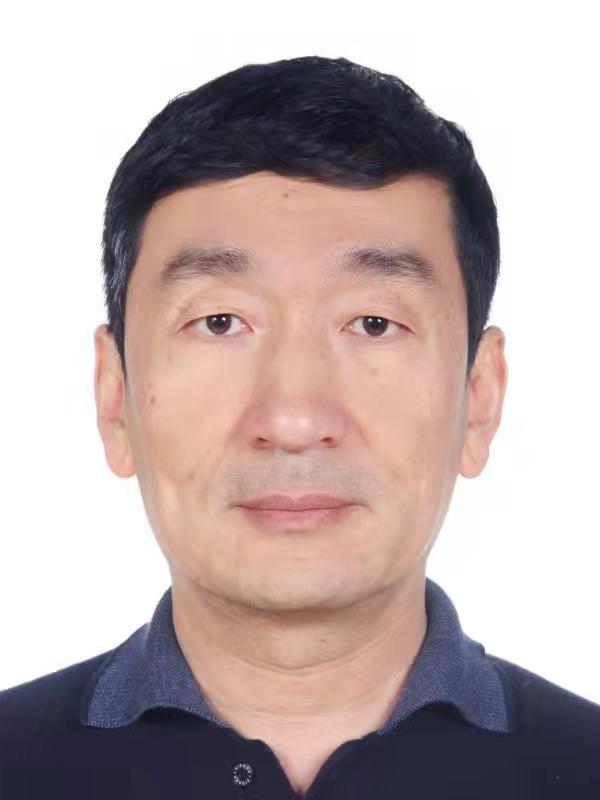}}]{Chao Zhang}
received the Ph.D. degree in Electrical Engineering from Beijing Jiaotong University, Beijing, China in 1995. He is currently a Research Professor at the Key Laboratory of Machine Perception (MOE), School of Intelligence Science and Technology, Peking University. His research interests include computer vision, image processing, machine learning, and pattern recognition.
\end{IEEEbiography}

\begin{IEEEbiography}[{\includegraphics[width=1in,height=1.25in,clip,keepaspectratio]{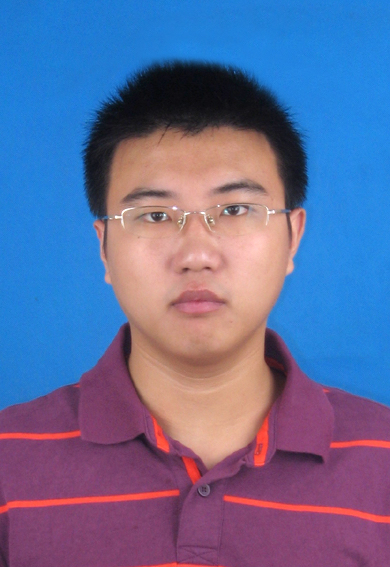}}]{Hongyang Zhang}
received the Ph.D. degree from Machine Learning Department, Carnegie Mellon University in 2019. He is currently an Assistant Professor at David R. Cheriton School of Computer Science, University of Waterloo, and a Faculty affiliated with Vector Institute for AI. Before joining University of Waterloo, he was a Postdoctoral Research Associate with Toyota Technological Institute at Chicago. His research interests include machine learning, AI security, and trustworthy AI.
\end{IEEEbiography}





\vfill


\clearpage

\begin{figure*}[t]
    \centering
    \begin{subfigure}[t]{\textwidth}
        \centering
        \includegraphics[width=\textwidth]{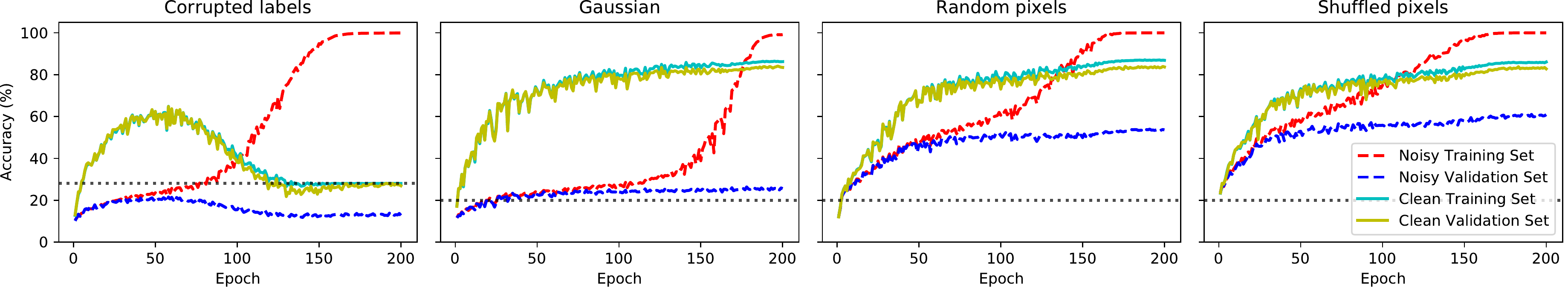}
        \caption{Accuracy curves of the model trained using ERM.}
        \label{fig:ce_acc_curve_r08}
    \end{subfigure}
    \vskip 0.1in
    \begin{subfigure}[t]{\textwidth}
        \centering
        \includegraphics[width=\textwidth]{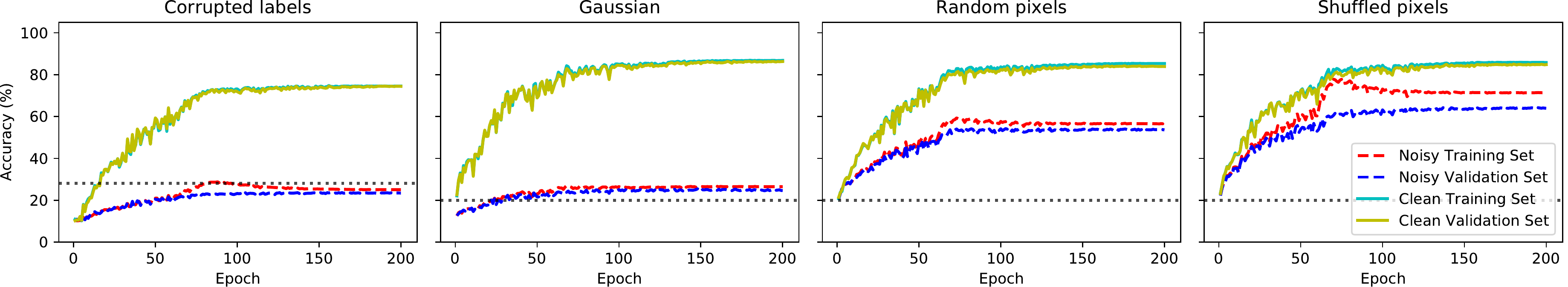}
        \caption{Accuracy curves of the model trained using our method.}
        \label{fig:wsc_acc_curve_r08}
    \end{subfigure}
    \caption{
    Accuracy curves of the model trained on the noisy CIFAR10 training set with a noise rate of 80\%. The horizontal dotted line displays the percentage of clean data in the training sets.
    It shows that our observations in Section~\ref{sec:approach} hold even when extreme label noise is injected.
    }
    \label{fig:acc_curve_r08}
\end{figure*}

\appendices

\setcounter{proposition}{0}
\setcounter{corollary}{0}

\section{Proofs}
\label{sec:proof}
\begin{proposition}
Let $d_{\mathrm{max}}$ be the maximal eigenvalue of the matrix $\X^{\intercal}\X$, if the learning rate $\eta < \frac{\alpha+1}{\alpha d_{\mathrm{max}}}$, then
\begin{align}
    \lim_{s\rightarrow\infty} \norm{\X\bm{\theta}^{(s)} - \bm{t}^{(s)}}_2^2 = 0.
\end{align}
\end{proposition}

\begin{proof}
Inserting Equation~\eqref{eq:theta_k} into \eqref{eq:t_k}, we have
\begin{align}
\begin{split}
\bm{t}^{(k)} & = \alpha\bm{t}^{(k-1)} + (1 - \alpha)\X\bm{\theta}^{(k)} \\
 & = \alpha\bm{t}^{(k-1)} \\
 &\quad~+ (1 - \alpha)\X[\bm{\theta}^{(k-1)} - \eta\X^{\intercal}(\X\bm{\theta}^{(k-1)} - \bm{t}^{(k-1)})] \\
 &= [\alpha\I+(1-\alpha)\eta\X\X^{\intercal}]\bm{t}^{(k-1)} \\
 &\quad~+ (1-\alpha)(\I - \eta\X\X^{\intercal})\X\bm{\theta}^{(k-1)}.
\label{eq:t_k_expand}
\end{split}
\end{align}

Note that $\X\X^{\intercal}$ is positive semi-definite and can be diagonalized as $\X\X^{\intercal} = \V^{\intercal}\D\V$, where the diagonal matrix $\D$ contains the eigenvalue of $\X^{\intercal}\X$ and the matrix $\V$ contains the corresponding eigenvectors, $\V\V^{\intercal}=\I$. And let $\bm{r}^{(k)}=\V\bm{t}^{(k)},\bm{s}^{(k)}=\V\X\bm{\theta}^{(k)}$. Multiplying both sides of Equation~\eqref{eq:t_k_expand} by $\V$, we have
\begin{align}
\begin{split}
\label{eq:r_k}
\V\bm{t}^{(k)} &= \V[\alpha\V^{\intercal}\V+(1-\alpha)\eta\V^{\intercal}\D\V]\bm{t}^{(k-1)} \\
               &\quad~+ (1-\alpha)\V(\V^{\intercal}\V - \eta\V^{\intercal}\D\V)\X\bm{\theta}^{(k-1)} \\
\V\bm{t}^{(k)} &= [\alpha\I+(1-\alpha)\eta\D]\V\bm{t}^{(k-1)} \\
    &\quad~+ (1-\alpha)(\I - \eta\D)\V\X\bm{\theta}^{(k-1)} \\
\bm{r}^{(k)} &= [\alpha\I+(1-\alpha)\eta\D]\bm{r}^{(k-1)} \\
    &\quad~+ (1-\alpha)(\I - \eta\D)\bm{s}^{(k-1)}.
\end{split}
\end{align}
From Equation~\eqref{eq:theta_k}, we have
\begin{align}
\begin{split}
\label{eq:s_k}
\bm{s}^{(k)} &= \V\X\bm{\theta}^{(k)} \\
 &= \V\X[\bm{\theta}^{(k-1)} - \eta\X^{\intercal}(\X\bm{\theta}^{(k-1)} - \bm{t}^{(k-1)})] \\
 &= \bm{s}^{(k-1)} - \eta\D(\bm{s}^{(k-1)} - \bm{r}^{(k-1)}) \\
 &= \eta\D\bm{r}^{(k-1)} + (\I - \eta\D)\bm{s}^{(k-1)}.
\end{split}
\end{align}
Subtracting the the both sides of Equation~\eqref{eq:s_k} by $\bm{r}^{(k)}$, we obtain
\begin{align}
\begin{split}
\label{eq:recurr}
\bm{s}^{(k)} - \bm{r}^{(k)} &= \alpha(\I - \eta\D)(\bm{s}^{(k-1)} - \bm{r}^{(k-1)}) \\
 &= [\alpha(\I - \eta\D)]^{k}(\bm{s}^{(0)} - \bm{r}^{(0)}) \\
 &= \A^{k}\V\bm{b},
\end{split}
\end{align}
where $\A=\alpha(\I - \eta\D), \bm{b}=\X\bm{\theta}^{(0)}-\bm{t}^{(0)}$. Therefore,
\begin{align}
\begin{split}
\X\bm{\theta}^{(k)} - \bm{t}^{(k)} &= \V^{\intercal}\A^{k}\V\bm{b}.
\end{split}
\end{align}
Because $\X\X^{\intercal}$ is positive semi-definite and $\alpha\in (0, 1)$, all elements in $\A=\alpha(\I - \eta\D)$ is smaller than 1. When $0 < \eta < \frac{\alpha+1}{\alpha d_{\mathrm{max}}}$, each elements of the diagonal matrix $\A$ is greater than -1, and we have
\begin{align}
\begin{split}
\lim_{k\rightarrow\infty} \norm{\X\bm{\theta}^{(k)} - \bm{t}^{(k)}}_2^2 &= \bm{b}^{\intercal}\V^{\intercal}\A^{2k}\V\bm{b} \\
    &= 0.
\end{split}
\end{align}

\end{proof}

\begin{corollary}
Under the same condition as in Proposition~\ref{prop:conver}, we have
\begin{align}
    \lim_{k\rightarrow\infty} \frac{\norm{\X\bm{\theta}^{(k+1)} - \bm{t}^{(k+1)}}_2^2}{\norm{\X\bm{\theta}^{(k)} - \bm{t}^{(k)}}_2^2} = a_i^2 < 1,
\end{align}
\end{corollary}

\begin{proof}
Let $a_i$ be the element of the diagonal matrix $\A$ that has the maximal absolute value, where $i$ is its index on $\A$, i.e., $i = \argmax_j |\A_{jj}|, |a_i| < 1$.
Then, we have
\begin{align}
\begin{split}
    & \lim_{k\rightarrow\infty} \frac{\norm{\X\bm{\theta}^{(k+1)} - \bm{t}^{(k+1)}}_2^2}{\norm{\X\bm{\theta}^{(k)} - \bm{t}^{(k)}}_2^2} \\
    = & \lim_{k\rightarrow\infty} \frac{\bm{b}^{\intercal}\V^{\intercal}\A^{2k+2}\V\bm{b}}{\bm{b}^{\intercal}\V^{\intercal}\A^{2k}\V\bm{b}} \\
    = & \frac{a_{i}^{2k+2}(\V\bm{b})_{i}^2}{a_{i}^{2k}(\V\bm{b})_{i}^2} \\
    = & a_i^2 < 1.
\end{split}
\end{align}
\end{proof}

\section{Experimental Setups}
\subsection{Double descent phenomenon}
\label{sec:setup_dd}
Following previous work~\cite{nakkiran2019deep}, 
we optimize all models using Adam~\cite{kingma2014adam} optimizer
with a fixed learning rate of 0.0001, a batch size of 128, the common data augmentations, and a weight decay of 0 for 4,000 epochs.
For our approach, we use the hyper-parameters $\mathrm{E}_s=40, \alpha=0.9$ for standard ResNet-18 (width of 64) and dynamically adjust them for other models according to the relation of model capacity $r=\frac{64}{\mathrm{width}}$ as:
\begin{equation}
\label{eq:dd_params}
    \mathrm{E}_s = 40\times r;\quad \alpha = 0.9^{\frac{1}{r}}.
\end{equation}

\subsection{Adversarial training}
\label{sec:setup_adv}
\cite{szegedy2013intriguing} reported that imperceptible small perturbations around input data (i.e., adversarial examples) can cause ERM-trained deep neural networks to make arbitrary predictions.
Since then, a large literature devoted to improving the adversarial robustness of deep neural networks.
Among them, the adversarial training algorithm TRADES \cite{zhang2019theoretically} achieves state-of-the-art performance.
TRADES decomposed robust error (w.r.t adversarial examples) to the sum of natural error and boundary error, and proposed to minimize:
\begin{equation}
\label{eq:trades_appendix}
\mathbb{E}_{\x,\y}\Bigg\{ \mathrm{CE}(\p(\x), \y) + \max_{\|\widetilde{\x}-\x\|_\infty\le\epsilon}\mathrm{KL}(\p(\x), \p(\widetilde{\x}))/\lambda\Bigg\},
\end{equation}
where $\p(\cdot)$ is the model prediction, $\epsilon$ is the maximal perturbation, CE stands for cross entropy, and KL stands for Kullback–Leibler divergence.
The first term corresponds to ERM that maximizes the natural accuracy; the second term pushes the decision boundary away from data points to improve adversarial robustness; the hyper-parameter $1/\lambda$ controls the trade-off between natural accuracy and adversarial robustness.
We evaluate self-adaptive training on this task by replacing the first term of Equation~\eqref{eq:trades_appendix} with our approach.

Our experiments are based on the official open-sourced implementation\footnote{\url{https://github.com/yaodongyu/TRADES}} of TRADES~\cite{zhang2019theoretically}.
Concretely, we conduct experiments on the CIFAR10 dataset \cite{krizhevsky2009cifar} and use WRN-34-10~\cite{zagoruyko2016wide} as base classifier.
For training, we use an initial learning rate of 0.1, a batch size of 128, and 100 training epochs. The learning rate is decayed at the 75th and 90th epoch by a factor of 0.1.
The adversarial example $\widetilde{\x}_i$ is generated dynamically during training by projected gradient descent (PGD) attack \cite{madry2017towards} with a maximal $\ell_{\infty}$ perturbation $\epsilon$ of 0.031, perturbation step size of 0.007, number of perturbation steps of 10.
The hyper-parameter $1/\lambda$ of TRADES is set to 6 as suggested by the original paper, $\mathrm{E}_s, \alpha$ of our approach is set to 70, 0.9, respectively.
For evaluation, we report robust accuracy $\frac{1}{n}\sum_i \mathbbm{1}\{ \argmax~p(\widetilde{\x}_i)=\argmax~\y_i \}$, where adversarial example $\widetilde{\x}$ is generated by two kinds of white box $\ell_{\infty}$ attacks with $\epsilon$ of 0.031: 1) AutoAttack~\cite{croce2020reliable} (as in Fig.~\ref{fig:robust_acc}); 2) untargeted PGD attack (as in Fig.~\ref{fig:robust_acc_pgd}), with a perturbation step size of 0.007, the number of perturbation steps of 20.

\begin{figure}[t]
    \centering
    \includegraphics[width=\linewidth]{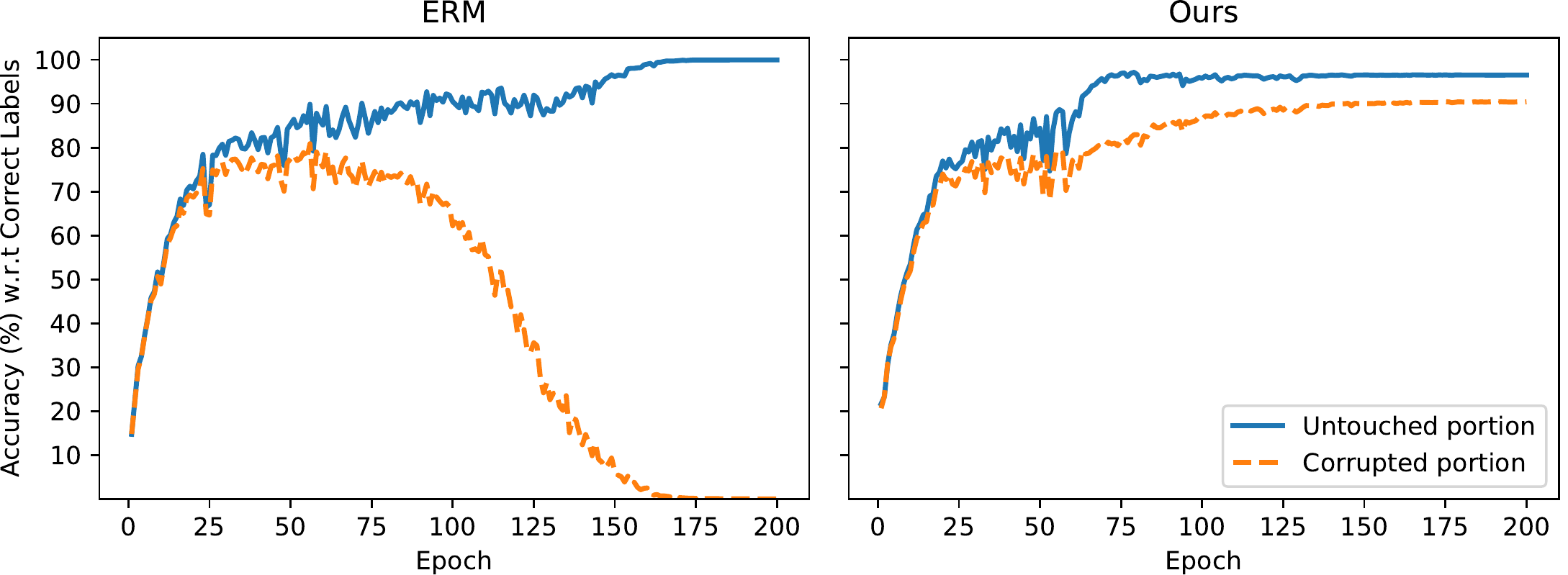}
    \caption{
    Accuracy curves on different portions of the CIFAR10 training set (with 40\% of label noise) w.r.t. correct labels. We split the training set into two portions:
    1)~\emph{Untouched portion}, i.e., the elements in the training set which were left untouched; 
    2)~\emph{Corrupted portion}, i.e., the elements in the training set which were indeed randomized. It shows that ERM fits correct labels in the first few epochs and then eventually overfits the corrupted labels. In contrast, self-adaptive training calibrates the training process and consistently fits the correct labels.
    }
    \label{fig:split_acc_curve_r04}
\end{figure}

\begin{figure*}[t]
\centering
\includegraphics[width=\textwidth]{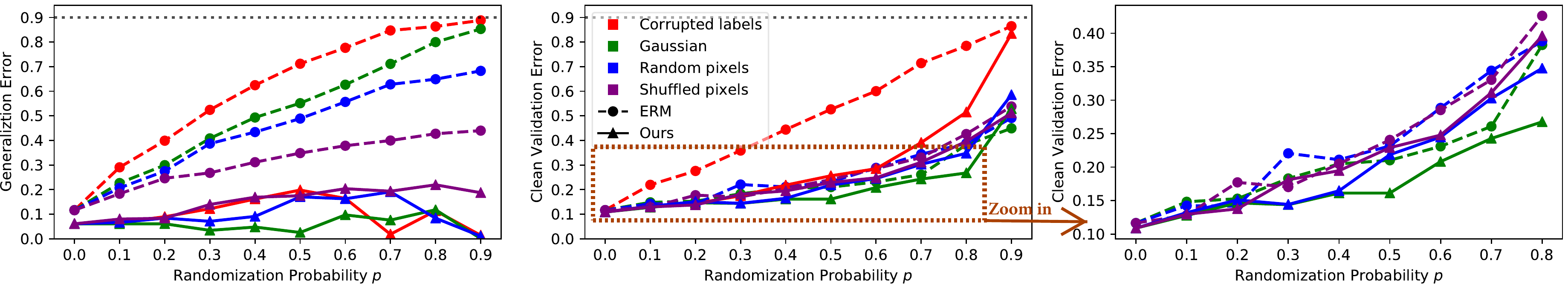}
\caption{
Generalization error and clean validation error under four kinds of random noise (represented by different colors) for ERM (the dashed curves) and our approach (the solid curves) on CIFAR10 when data augmentation is turned off. We zoom in on the dashed rectangle region and display it in the third column for a clear demonstration.
}
\label{fig:gen_clean_errs_wo_aug}
\end{figure*}

\subsection{Supervised learning on ImageNet}
\label{sec:setup_imagenet}
We use ResNet-50/ResNet-101~\cite{he2016deep} as base classifier.
Following the original papers~\cite{he2016deep} and~\cite{loshchilov2016sgdr,goyal2017accurate}, we use SGD to optimize the networks with a batch size of 768, a base learning rate of 0.3, a momentum of 0.9, a weight decay of 0.0005, and a total training epoch of 95.
The learning rate is linearly increased from 0.0003 to 0.3 in the first 5 epochs (i.e., warmup), and then decayed using the cosine annealing schedule~\cite{loshchilov2016sgdr} to 0.
Following common practice, we use the random resizing, cropping, and
flipping augmentations during training.
The hyper-parameters of our approach are set to $\mathrm{E}_s = 50$ and $\alpha=0.99$ under standard setup, and are set to $\mathrm{E}_s = 60$ and $\alpha = 0.95$ under 40\% label noise setting. The experiments are conducted on PyTorch~\cite{paszke2019pytorch} with distributed training and mixed-precision training\footnote{\url{https://github.com/NVIDIA/apex}} for acceleration.

\subsection{Linear classifier during self-supervised training}
\label{sec:setup_online_lin_cls}
For the horizontal line in Fig.~\ref{fig:ssl_acc_at_training}, we randomly initialize a network and train a linear classifier atop this network. The horizontal line corresponds to the \emph{final} accuracy of this linear classifier. For the other three curves in Fig.~\ref{fig:ssl_acc_at_training}, we adopt an online classifier scheme, following the official implementation of BOYL~\cite{grill2020bootstrap}. Concretely, in the self-supervised pre-training stage, we train 1) a network to fit the output of another randomly initialized network, and 2) a linear classifier on top of the network whose gradient will not backpropagate to the network. This scheme bypasses the heavy cost to train a linear classifier from scratch at each training epoch and allows us to directly evaluate the linear classifier on the validation set. In practice, we use a separate learning rate for this classier, i.e., 0.4 in our experiments.

\subsection{Training time comparison}
\label{sec:setup_ssl_training_time}
In this comparison, we calculate the running time for all methods using a server with 8 V100 GPUs, a 40-core CPU, CUDA10.1, and PyTorch1.8. For MoCov2, we follow the official implementation using a batch size of 256 and ShuffleBN~\cite{he2020momentum}. For the BYOL, we use SyncBn and a batch size of 512 to fit into a single machine. For our method, we use SyncBN and a batch size of 1024 so that the number of augmented views (1024x1) in each batch is identical to that of BYOL (512x2). For the supervised cross entropy baseline, we use the standard BN and a batch size of 256 following the common practice. According to Fig.~\ref{fig:batch_time}, we can see that, though data preprocessing is not the bottleneck on a modern machine, our method is still ~2x/1.3x faster than BYOL/MoCov2.

\section{Additional Experimental Results}
\label{sec:extra_exp}
\subsection{ERM may suffer from overfitting of noise}

In~\cite{zhang2016understanding}, the authors showed that
the model trained by standard ERM can easily fit randomized data. However, they only analyzed the generalization errors in the presence of corrupted labels. In this paper, we report the whole training process and also consider the performance on clean sets (i.e., the original uncorrupted data). Fig.~\ref{fig:ce_acc_curve} shows the four accuracy curves (on clean and noisy training, validation set, respectively) for each model that is trained on one of four corrupted training data. Note that the models can only have access to the noisy training sets (i.e., the red curve), and the other three curves are shown only for illustration purposes. We conclude with two principal observations from the figures:
(1) The accuracy on noisy training and validation sets is close at the beginning and the gap is monotonously increasing w.r.t. epoch. The generalization errors (i.e., the gap between the accuracy on noisy training and validation sets) are large at the end of training.
(2) The accuracy on the clean training and validation sets is consistently higher than the percentage of clean data in the noisy training set. This occurs around the epochs between underfitting and overfitting.

\begin{figure}[t]
\centering
\includegraphics[width=.8\linewidth]{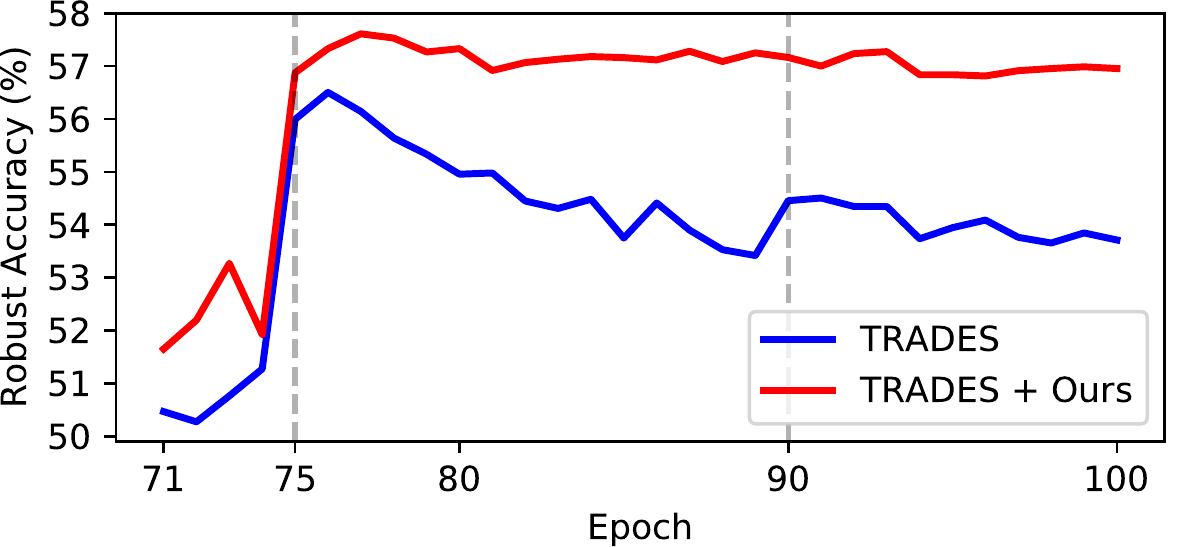}
\caption{
Robust Accuracy (\%) on CIFAR10 test set under white box $\ell_\infty$ PGD-20 attack ($\epsilon$=0.031). The vertical dashed lines indicate learning rate decay. It shows that self-adaptive training consistently improves TRADES.
}
\label{fig:robust_acc_pgd}
\end{figure}

Our first observation poses concerns on the overfitting issue of ERM training dynamic which has also been reported by~\cite{li2019gradient}. However, the work of~\cite{li2019gradient} only considered the case of corrupted labels and proposed using the early-stop mechanism to improve the performance on the clean data. On the other hand, our analysis of the broader corruption schemes shows that the early stopping might be sub-optimal and may hurt the performance under other types of corruption (see the last three columns in Fig.~\ref{fig:ce_acc_curve}).

The second observation implies that model predictions by ERM can capture and amplify useful signals in the noisy training set, although the training dataset is heavily corrupted. While this was also reported in~\cite{zhang2016understanding,rolnick2017deep,guan2018said,li2019gradient} for the case of corrupted labels, we show that a similar phenomenon occurs under other kinds of corruption more generally. This observation sheds light on our approach, which incorporates model predictions into the training procedure.

\subsection{Improved generalization of self-adaptive training on random noise}

\noindent\textbf{Training accuracy w.r.t. correct labels on different portions of data}\quad
For a more intuitive demonstration, we split the CIFAR10 training set (with 40\% label noise) into two portions:
1)~\emph{Untouched portion}, i.e., the elements in the training set which were left untouched;
2)~\emph{Corrupted portion}, i.e., the elements in the training set which were indeed randomized.
The accuracy curves on these two portions w.r.t correct training labels are shown in Fig.~\ref{fig:split_acc_curve_r04}.
We can observe that the accuracy of ERM on the corrupted portion first increases in the first few epochs and then eventually decreases to 0. In contrast, self-adaptive training calibrates the training process and consistently fits the correct labels.

\begin{figure}[t]
    \centering
    \includegraphics[width=.8\linewidth]{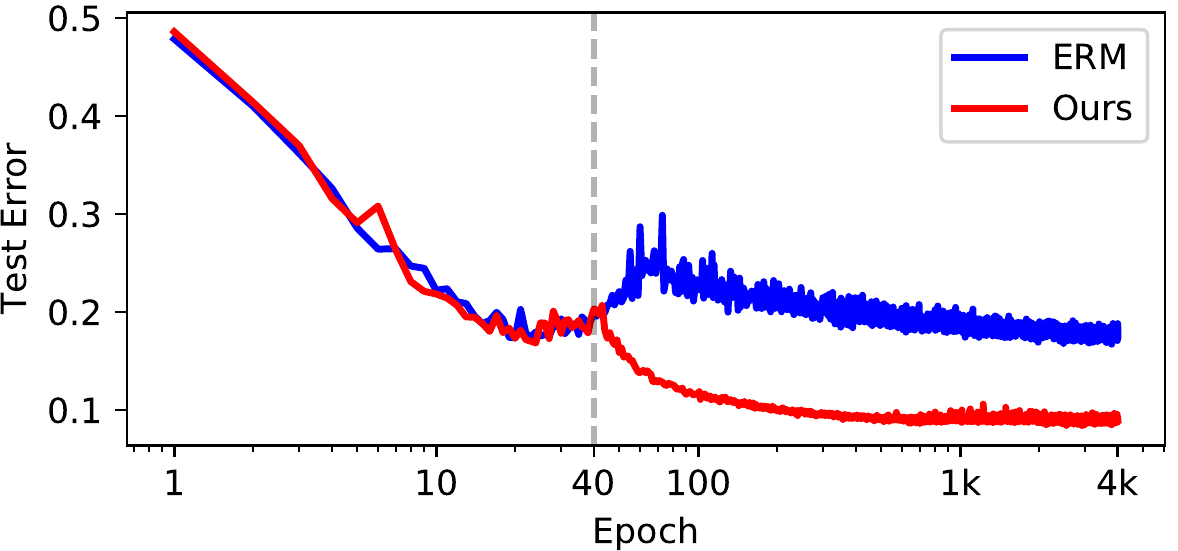}
    \caption{Self-adaptive training \emph{vs.} ERM on the error-epoch curve.
    We train the standard ResNet-18 networks (i.e., the width of 64) on the CIFAR10 dataset with 15\% randomly-corrupted labels and report the test errors on the clean data.
    The dashed vertical line represents the initial epoch $\mathrm{E}_s$ of our approach.
    It shows that self-adaptive training has a significantly diminished epoch-wise double-descent phenomenon.
    }
    \label{fig:epochwise_dd}
\end{figure}

\begin{table*}[t]
\caption{Test Accuracy (\%) on CIFAR datasets with various levels of uniform label noise injected into the training set.
We show that considerable gains can be obtained when combined with SCE loss.}
\label{tab:noisy_cls_ours_sce}
\begin{center}
\begin{small}
\begin{tabular}{lcccccccc}
\toprule
 & \multicolumn{4}{c}{CIFAR10} & \multicolumn{4}{c}{CIFAR100} \\
 \cmidrule(l{3pt}r{3pt}){2-5} \cmidrule(l{3pt}r{3pt}){6-9}
Label Noise Rate  & 0.2 & 0.4 & 0.6 & 0.8 & 0.2 & 0.4 & 0.6 & 0.8   \\
\midrule
SCE~\cite{wang2019symmetric} & 90.15 & 86.74 & 80.80 & 46.28 & 71.26 & 66.41 & 57.43 & 26.41\\
Ours    & 94.14 & 92.64 & 89.23 & 78.58 & 75.77 & 71.38 & 62.69 & 38.72 \\
Ours + SCE & \textbf{94.39} & \textbf{93.29} & \textbf{89.83} & \textbf{79.13} & \textbf{76.57} & \textbf{72.16} & \textbf{64.12} & \textbf{39.61} \\
\bottomrule
\end{tabular}
\end{small}
\end{center}
\end{table*}

\medskip
\noindent\textbf{Study on extreme noise}\quad
We further rerun the same experiments as in Fig.~\ref{fig:acc_curve} of the main text by injecting extreme noise (i.e., noise rate of 80\%) into the CIFAR10 dataset. We report the corresponding accuracy curves in Fig.~\ref{fig:acc_curve_r08}, which shows that our approach significantly improves the generalization over ERM even when random noise dominates training data. This again justifies our observations in Section~\ref{sec:approach}.

\medskip
\noindent\textbf{Effect of data augmentation}\quad
All our previous studies are performed with common data augmentation (i.e., random cropping and flipping). Here, we further report the effect of data augmentation. We adjust the introduced hyper-parameters as $\mathrm{E}_s=25$, $\alpha=0.7$ due to severer overfitting when data augmentation is absent. Fig.~\ref{fig:gen_clean_errs_wo_aug} shows the corresponding generalization errors and clean validation errors. We observe that, for both ERM and our approach, the errors clearly increase when data augmentation is absent (compared with those in Fig.~\ref{fig:gen_clean_errs}). However, the gain is limited and the generalization errors can still be very large, with or without data augmentation for standard ERM. Directly replacing the standard training procedure with our approach can bring bigger gains in terms of generalization regardless of data augmentation. This suggests that data augmentation can help but is not of the essence to improving the generalization of deep neural networks, which is consistent with the observation in~\cite{zhang2016understanding}.

\subsection{Epoch-wise double descent phenomenon}
\cite{nakkiran2019deep} reported that, for a sufficiently large model, the test error-training epoch curve also exhibits a double-descent phenomenon,
which they termed \emph{epoch-wise double descent}.
In Fig.~\ref{fig:epochwise_dd}, we reproduce the epoch-wise double descent phenomenon on ERM
and inspect self-adaptive training.
We observe that our approach (the red curve) exhibits a slight double-descent
due to the overfitting starts before the initial $\mathrm{E}_s$ epochs.
As the training targets being updated (i.e., after $\mathrm{E}_s$ = 40 training epochs),
the red curve undergoes a monotonous decrease.
This observation again indicates that the double-descent phenomenon 
may stem from the overfitting of noise and can be avoided by our algorithm.

\subsection{Cooperation with Symmetric Cross Entropy}
\label{sec:sat_sce}
\cite{wang2019symmetric} showed that Symmetric Cross Entropy (SCE) loss
is robust to underlying label noise in training data.
Formally, given training target $\bm{t}_i$ and model prediction $\p_i$,
SCE loss is defined as:
\begin{equation}
    \mathcal{L}_{sce} = -w_1 \sum_j \bm{t}_{i,j}~\log~\p_{i,j} - w_2 \sum_j \bm{p}_{i,j}~\log~\bm{t}_{i,j},
\end{equation}
where the first term is the standard cross entropy loss
and the second term is the reversed version.
In this section, we show that self-adaptive training can cooperate with
this noise-robust loss and enjoy further performance boost without extra cost.

\medskip
\noindent\textbf{Setup}\quad
Most of the experimental settings are kept the same as Section~\ref{sec:exp_label_noise}.
For the introduced hyper-parameters $w_1, w_2$ of SCE loss,
we directly set them to 1, 0.1, respectively, in all our experiments.

\medskip
\noindent\textbf{Results}\quad
We summarize the results in Table~\ref{tab:noisy_cls_ours_sce}. 
We can see that, although self-adaptive training already achieves very strong performance, considerable gains can be obtained when equipped with SCE loss. Concretely, the improvement is as large as 1.5\% when 60\% of label noise is injected into the CIFAR100 training set.
It also indicates that our approach is flexible and can be further extended.

\begin{table}[t]
\caption{Average Accuracy (\%) on the CIFAR10 test set and out-of-distribution dataset CIFAR10-C at various corruption levels.}
\label{tab:cifar10c}
\begin{center}
\begin{small}
\begin{tabular}{lcccccc}
\toprule
\multirow{2}{*}{Method} & \multirow{2}{*}{CIFAR10} & \multicolumn{5}{c}{Corruption Level@CIFAR10-C} \\
\cmidrule{3-7}
 & & 1 & 2 & 3 & 4 & 5 \\
\midrule
ERM     & 95.32 & 88.44 & 83.22 & 77.26 & 70.40 & 58.91 \\
Ours    & \textbf{95.80} & \textbf{89.41} & \textbf{84.53} & \textbf{78.83} & \textbf{71.90} & \textbf{60.77} \\
\bottomrule
\end{tabular}
\end{small}
\end{center}
\end{table}

\subsection{Out-of-distribution generalization}
In this section, we consider the out-of-distribution (OOD) generalization, where the models are evaluated on unseen test distributions outside the training distribution. 

\medskip
\noindent\textbf{Setup}\quad
To evaluate the OOD generalization performance, we use the CIFAR10-C benchmark~\cite{hendrycks2018benchmarking} that is constructed by applying 15 types of corruption to the original CIFAR10 test set at 5 levels of severity. The performance is measured by average accuracy over 15 types of corruption. We mainly follow the training details in Section~\ref{sec:exp_label_noise} and adjust $\alpha=0.95, \mathrm{E}_s=80$.

\medskip
\noindent\textbf{Results}\quad
We summarize the results in Table~\ref{tab:cifar10c}. Regardless of the presence of corruption and corruption levels, our method consistently outperforms ERM by a considerable margin, which becomes large when the corruption is more severe. The experiment indicates that self-adaptive training may provide implicit regularization for OOD generalization.

\end{document}